\documentclass[authoryear]{elsarticle}

\journal{Data Mining and Knowledge Discovery (Springer)}

\usepackage{hyperref}
\hypersetup{hidelinks,citecolor={blue},linkcolor={black}}

\usepackage{geometry}

\usepackage[x11names]{xcolor}

\newif\ifproofread

\usepackage[utf8]{inputenc}
\usepackage{enumitem}

\usepackage{wrapfig}
\usepackage{url}

\usepackage[export]{adjustbox}

\usepackage{natbib}

\usepackage{subfig}
\usepackage{graphicx} 
\usepackage[misc,geometry]{ifsym}
\usepackage{amsmath}
\usepackage{amsthm}
\usepackage{setspace}

\usepackage{amssymb}
\usepackage{epstopdf}
\usepackage{booktabs}
\usepackage{tabularx}
\usepackage{color}
\usepackage{float}
\usepackage{rotating}
\usepackage{dblfloatfix}
\usepackage{multirow}
\usepackage{multicol}
\usepackage{moresize}
\usepackage[flushleft]{threeparttable}
\usepackage[ruled,vlined,linesnumberedhidden]{algorithm2e}

\usepackage{mathtools}
\DeclarePairedDelimiter{\ceil}{\lceil}{\rceil}

\newsavebox\CBox
\def\textBF#1{\sbox\CBox{#1}\resizebox{\wd\CBox}{\ht\CBox}{\textbf{#1}}}
\usepackage{tikz}
\usetikzlibrary{arrows}

\newtheorem{theorem}{Theorem}
\newtheorem{corollary}{Corollary}

\def\textBF#1{\sbox\CBox{#1}\resizebox{\wd\CBox}{\ht\CBox}{\textbf{#1}}}
\newcommand\MyBox[2]{
  \fbox{\lower0.75cm
    \vbox to 1.7cm{\vfil
      \hbox to 1.7cm{\hfil\parbox{1.4cm}{#1\\#2}\hfil}
      \vfil}%
  }%
}

\newcommand\MyBoxx[1]{
  \fbox{\lower0.75cm
    \vbox to 1.7cm{\vfil
      \hbox to 1.7cm{\hfil\parbox{0.46cm}{#1}\hfil}
      \vfil}%
  }%
}

\begin{document}

\proofreadtrue

\begin{frontmatter}

\title{The Area Under the ROC Curve as a Measure of Clustering Quality}

\author[pablo]{Pablo~A.~Jaskowiak\corref{corr}}
\cortext[corr]{Corresponding author}
\ead{pablo.andretta@ufsc.br}

\author[ivan]{Ivan G. Costa}

\author[ricardo1]{Ricardo~J.~G.~B.~Campello}

\fntext[pablo]{Federal University of Santa Catarina, Joinville, Santa Catarina, Brazil.}

\fntext[ivan]{Institute for Computational Genomics, RWTH Aachen University Medical Faculty, Aachen, Germany.}

\fntext[ricardo1]{School of Mathematical and Physical Sciences, University of Newcastle, Australia.}


\begin{abstract}

The Area Under the the Receiver Operating Characteristics (ROC) Curve, referred to as AUC, is a well-known performance measure in the supervised learning domain. Due to its compelling features, it has been employed in a number of studies to evaluate and compare the performance of different classifiers. In this work, we explore AUC as a performance measure in the unsupervised learning domain, more specifically, in the context of cluster analysis. In particular, we elaborate on the use of AUC as an internal/relative measure of clustering quality, which we refer to as Area Under the Curve for Clustering (AUCC). We show that the AUCC of a given candidate clustering solution has an expected value under a null model of random clustering solutions, regardless of the size of the dataset and, more importantly, regardless of the number or the (im)balance of clusters under evaluation. In addition, we elaborate on the fact that, in the context of internal/relative clustering validation as we consider, AUCC is actually a linear transformation of the Gamma criterion from~\cite{BakHub75}, for which we also formally derive a theoretical expected value for chance clusterings. We also discuss the computational complexity of these criteria and show that, while an ordinary implementation of Gamma can be computationally prohibitive and impractical for most real applications of cluster analysis, its equivalence with AUCC actually unveils a much more efficient algorithmic procedure.~Our theoretical findings are supported by experimental results. These results show that, in addition to an effective and robust quantitative evaluation provided by AUCC, visual inspection of the ROC curves themselves can be useful to further assess a candidate clustering solution from a broader, qualitative perspective as well.  
\end{abstract}

\end{frontmatter}

\doublespacing

\setcounter{footnote}{0}
\clearpage
\section{Introduction}~\label{intro}

The introduction of Receiver Operating Characteristics (ROC) to the machine learning community is often attributed to the work of~\citet{Spackman1989}. Since then, ROC analysis has gained popularity in the supervised learning domain, in part as a result of the drawbacks observed with accuracy-based evaluations of classifiers~\citep{Bradley97,Provost97,Provost1998,Huang2005,Fawcett06,Flach2010}, especially for class imbalanced problems. Currently, ROC analysis stands as a valuable tool to visualize, evaluate and compare the performance of different classifiers~\citep{Majnik2013,Orallo2013}.

Given a classifier and a dataset for which desired classification outcomes (i.e., actual class labels) are available, the first step towards performing a ROC analysis consists in deriving statistics that relate classifier predictions with the corresponding desired outcomes. In the case of a binary classification problem with a positive and a negative class, classifier predictions can be deemed True Positive~($TP$), False Positive ($FP$), True Negative~($TN$), or False Negative~($FN$) with respect to actual class labels.
%
%
If the classifier under evaluation produces as output a probability or a score for each object, representing its likelihood or degree of membership to a class, a ROC curve can be derived by plotting the values of False Positive Rate ($FPR = FP / N$) against those of True Positive Rate~($TPR = TP / P$), where $P$ and $N$ stand for the cardinality of the positive and negative classes, respectively. In this case, each point in the curve is associated with a classification threshold, which lies within the interval of the scores produced by the classifier. For each threshold value, objects are deemed as positive or negative according to their classification score relative to the threshold, and the respective TPR and FPR values are calculated and plotted, as illustrated in Fig.~\ref{fig:roc}(a). The diagonal line in this figure ($TPR = FPR$) accounts for the expected performance of a random classifier. A classifier with a curve close to the top-left corner of the graph is usually preferred, whereas a classifier with a curve below the diagonal line performs worse than random. By reversing the classification predictions, its new ROC curve will be mirrored around the diagonal line. For detailed discussions on ROC graphs, see e.g. \citep{Fawcett2004,Fawcett06,Flach2010}.

\begin{figure}[htpb]
\centering
\begin{sideways} \hspace{0.55cm} \scriptsize{True Positive Rate (TPR)} \end{sideways}\hspace{-0.05cm}
\includegraphics[keepaspectratio=true,width=0.3\linewidth]{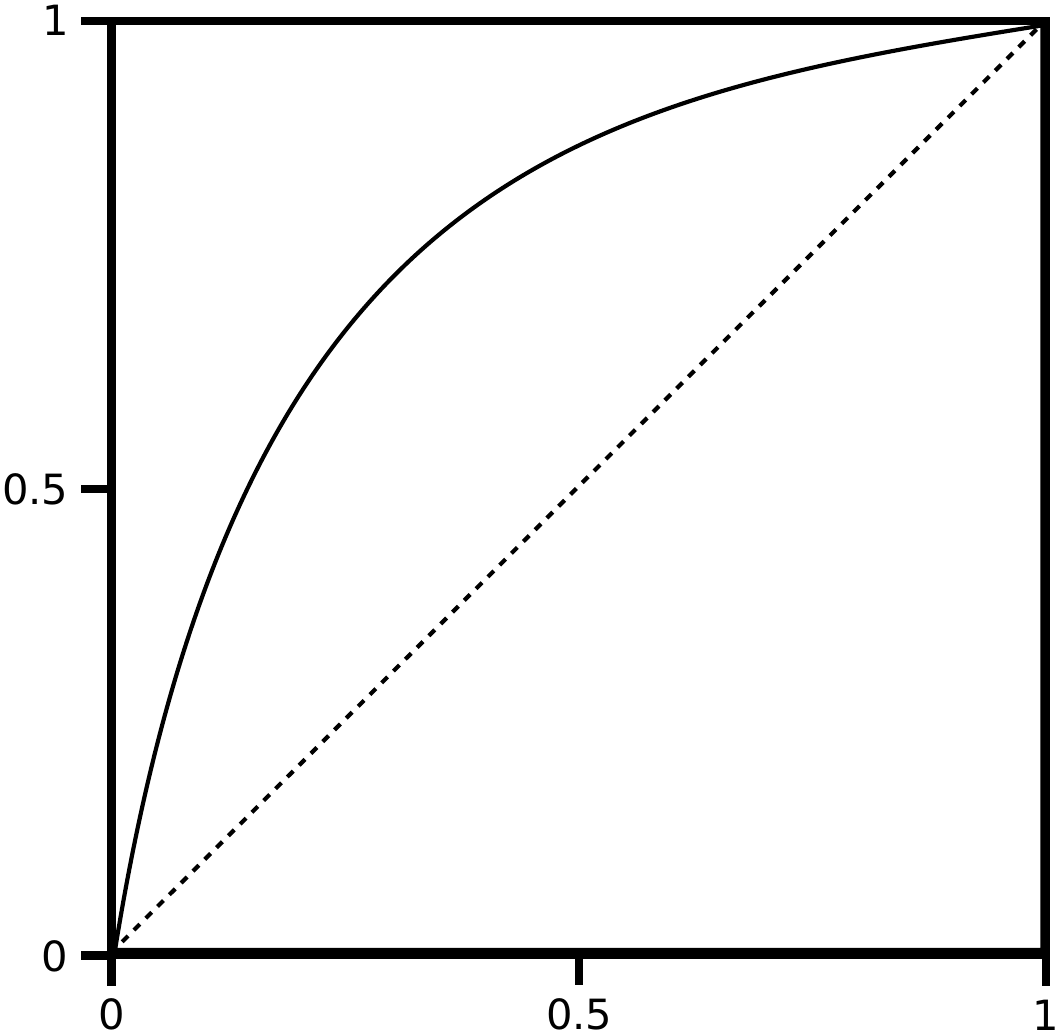}\hspace{2cm}
\begin{sideways} \hspace{0.55cm} \scriptsize{True Positive Rate (TPR)} \end{sideways}\hspace{-0.05cm}
\includegraphics[keepaspectratio=true,width=0.3\linewidth]{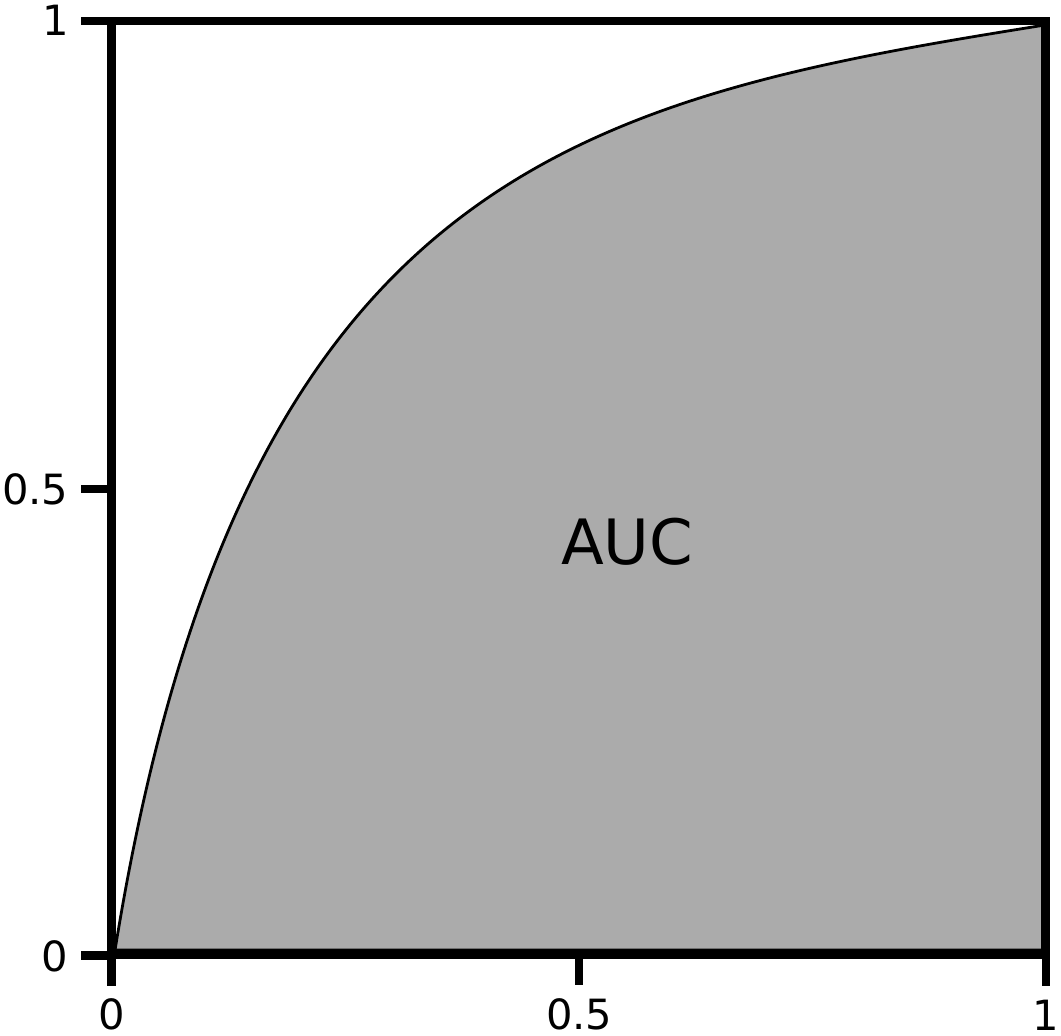}
\vspace*{-0.2cm}
\begin{flushleft} \hspace{2.55cm} \scriptsize{False Positive Rate (FPR)} \hspace{3.35cm} \scriptsize{False Positive Rate (FPR)}\end{flushleft}
\begin{flushleft} \hspace{3.92cm} (a) \hspace{6.2cm} (b)\end{flushleft}
\vspace{-0.3cm}
\caption{Examples of ROC graphs. A ROC curve (a) can be obtained from a scoring classifier. The dotted line in (a) corresponds to a classifier with performance comparable to that of a completely random classifier. In (b), the Area Under the Curve (AUC) of the corresponding curve is highlighted.}
\label{fig:roc}
\end{figure}

Even though the ROC graph has visual appeal, an aggregated scalar value is typically obtained, in order to make the comparison among different classifiers straightforward. A well-known choice, which is typically regarded as the most important statistics derived from ROC curves~\citep{Flach2010}, is the Area Under the (ROC) Curve~(AUC), as illustrated in Fig.~\ref{fig:roc}(b). The AUC of a classifier/model prediction consists of a single value in the~$[0,1]$ interval that, from a statistical standpoint, can be regarded as the probability that it will rank (or score) a randomly selected positive object higher than a randomly selected negative one~\citep{Fawcett06,Flatch2011}. From this observation it follows that, in general: (i) the larger the AUC value, the better is the performance of the classifier under evaluation; (ii) values of AUC around 0.5 indicate the expected performance of random classifiers\footnote{In fact, non-random classifiers can also exhibit such a performance~\citep{Flach2010}.}; and (iii) values below 0.5 indicate a worse than random classifier. 

Over the years, AUC became one of the standard measures employed to assess classifiers' performance. Nowadays, it is widely accepted that AUC should be favored over accuracy, based on both theoretical and empirical evidence. For instance, \citet{Huang2005} compared the evaluations obtained with these two measures regarding their consistency and discriminant power. Their results suggest that both measures have a high degree of consistency in their evaluations, i.e., they do not contradict each other, but AUC has a better discriminant power, i.e., it can discriminate between classification models when accuracy cannot.

In this paper we take a different perspective on AUC, by considering it in the unsupervised learning domain. More specifically, we elaborate on the use of AUC as an internal/relative measure of clustering quality~\citep{JaiDub88,Xu2009,Hennig2015}, which can be employed to evaluate and compare the results obtained from different clustering algorithms or parameterizations of a particular clustering algorithm. Hereafter we shall refer to this measure as Area Under the Curve for Clustering, or simply AUCC. The concept of AUC has been previously considered in the clustering scenario by \citet{JasCamCos12,JasCamCos13} and \citet{GiaBosPinUtr13}. In both cases, however, the AUC was employed within a limited scope, namely, to evaluate the agreement between proximity measures and the \emph{external labels} of a dataset. The goal was to evaluate proximity measures for clustering gene expression microarray data and did not include any theoretical or experimental evaluation of AUC, let alone as an internal/relative criterion.

In contrast, the AUCC measure studied in this paper is defined as an internal, relative clustering evaluation criterion. It operates on the set of all pairs of data objects, with their similarity playing the role of ``classification scores'' and whether or not they belong to the same cluster (in a candidate clustering solution to be assessed) playing the role of ``binary class labels''. In this setting, not only the area under the ROC curve can be computed as a quantitative measure of clustering quality, but visual-based, qualitative interpretation of the ROC curves themselves can also be undertaken.

As one of the main contributions in this paper, we theoretically show that the AUCC of a clustering solution has the same expected value as in classification (0.5) under the assumption of a relevant null model of random clusterings, regardless of the number of clusters or relative cluster sizes in the partitions under evaluation. As we will show, this is particularly desirable in scenarios comparing clustering results with different numbers of clusters. It is worth remarking that, while the study of the expected behavior of evaluation criteria for random clusterings (and the correction of their possible biases, commonly referred to as \emph{adjustment for chance}) has been extensively studied in the context of \emph{external} clustering evaluation, where a ground-truth clustering solution consisting of external labels exist --- see e.g. \citep{Romano2016}, this aspect has been surprisingly overlooked in the context of internal/relative (i.e. fully unsupervised) evaluation.

In addition to the above, we explore and theoretically elaborate on the result --- originally and preliminarily described in \citep{JaskowiakPhD2015} --- that the quantity defined here as AUCC is actually a linear transformation of the Gamma internal clustering validity criterion introduced more than 40 years ago by~\citet{BakHub75}, for which we also formally derive a theoretical expected value for chance clusterings. Strategies to handle ties in (dis)similarity values are discussed for both AUCC and Gamma in light of their linear relationship as well as their theoretical expected values.

The Gamma criterion was the best performer in a previous evaluation study by~\citet{Mil81}, in which 30 internal clustering validity criteria were assessed on the basis of their (dis)agreement --- as measured by correlation --- with external evaluations. Notwithstanding, Gamma's computational time includes a prohibitive $O(n^4)$ term, where $n$ is the number of data objects (dataset size). We show that AUCC has a significantly lower computational complexity, making it computationally tractable in real world problems involving datasets of practical relevance. This allowed us to carry out an empirical evaluation of AUCC in the context of relative clustering validation, relating its performance to that of 28 other commonly employed relative measures from the clustering literature.

The remainder of the paper is organized as follows. In Section~\ref{validation} we provide a brief overview of performance evaluation in cluster analysis, which is commonly referred to as clustering validity or validation. In Section~\ref{rocclus} we introduce AUCC as an internal/relative validity criterion for the unsupervised evaluation of clustering results. We then show that AUCC is equivalent to a linear transformation of the Gamma Index and that, under a null/random model assumption, both AUCC and Gamma have expected values that do not depend on the number (or the sizes) of clusters under evaluation. The section concludes with a discussion of how to handle ties in (dis)similarity values while preserving the theoretical findings in the paper.~An empirical study is carried out in Section~\ref{eval}, involving both quantitative evaluations provided by AUCC as well as qualitative evaluations based on visual inspection of ROC curves. Final remarks and conclusions are drawn in Section~\ref{conc}.

\section{Performance Evaluation in Cluster Analysis}\label{validation}

Most clustering algorithms from the literature will produce an output regardless of the existence of actual clusters in the data. Even if one assumes that clusters exist, their number and distributions are usually unknown. In order to avoid the use of spurious (i.e., meaningless or poor) clustering results, one can resort to clustering validation techniques \citep{Hennig2015}. According to \citet{JaiDub88}, clustering validation can be defined as the set of tools and procedures that are used in order to evaluate clustering results in a quantitative and objective manner.

Clustering validation techniques can be broadly divided into external, internal, and relative~\citep{JaiDub88,HalBatVaz01}. External criteria are mostly employed in the evaluation of clustering results against a desired clustering solution known beforehand~(i.e., a ground-truth). Although they are very useful for algorithm evaluation and comparison in controlled experiments, external criteria have limited applicability in practical cluster analysis scenarios, where a ground-truth doesn't exist \citep{FaeGueKriKroetal10,JasMouFurCam15}. Internal validity criteria rely their evaluation only on clustering assignments and the data themselves.~Internal criteria that are also relative can be used to assess and compare the quality of different partitions in a relative manner. For this reason, relative criteria are frequently employed in practical clustering applications, helping the selection of a final clustering solution for further inspection by a field practitioner.

The literature on relative validity criteria is extensive and a large number of measures have been proposed. These are usually conceived based on the idea that a good clustering solution (partition) should have compact and separate clusters~\citep{HalBatVaz01}. From different definitions of cluster compactness and separation, different relative validity measures arise. Back in the 1980's, \citet{Mil81,MilCoo85} compared the performance of 30 validity criteria, mostly relative ones. Since then, new measures have been introduced, e.g., \citet{Rou87,BezPal98,HalVaz08,MouJasCamZim14}, extensive reviews, assessments and evaluations of those have been performed, e.g., \citet{Maulik02,VenCamHru09,VenCamHru10,Arbelaitz2013}, and different implementations of the measures have been made available, e.g., \citet{clValid2008,NbClust,ClusterCrit2016}.

\section{The Area Under the Curve as an Internal/Relative Measure}\label{rocclus}

Consider a dataset $\mathbf{X} = \{\mathbf{x}_1,\dots,\mathbf{x}_n\}$ with $n$ objects embedded in a space where a measure of similarity  between pairs of objects can be defined (e.g., an Euclidean space with $d$ dimensions, i.e., $\mathbf{x}_i=\{x_{i1},\dots,x_{id}\}$ for $i = 1, \dots, n$). In addition, consider a clustering result in the form of a partition, that is, a labeling of the data into $2 \le k \le n-1$ mutually exclusive clusters.~Let~$\mathcal{C} = \{C_1,\dots,C_k\}$ denote this partition, with the following properties:
\begin{gather*}
{C_1} \cup \dots \cup {C_k} = \mathbf{X}, \\
{C_i} \neq \emptyset, \forall i,\\
{C_i} \cap {C_j} = \emptyset, \forall i,j \text{ with } i \neq j.
\end{gather*}
Notice that we can transform any clustering solution $\mathcal{C}$ as above into a \emph{pairwise} representation $\mathcal{C}^p$~(a binary relation) composed of $n(n-1)/2$ elements (object pairs), as follows:
\begin{displaymath}
\label{adap2}
\mathcal{C}^p(\mathbf{x}_i,\mathbf{x}_j) = \left\{ 
\begin{array}{l l}
  1 & \quad \mbox{if } \exists l: \mathbf{x}_i,\mathbf{x}_j \in C_l,\\ \noalign{\medskip}
  0 & \quad \mbox{otherwise}.\\
\end{array} \right.
\end{displaymath}
\noindent Let $\mathbf{D}$ be a pairwise similarity matrix of the objects in dataset $\mathbf{X}$, from which a clustering solution $\mathcal{C}$ to be evaluated was derived. The binary relation $\mathcal{C}^p$ of $\mathcal{C}$ can be used, along with the pairwise similarities $\mathbf{D}$, as input to ROC analysis. The rationale behind this type of evaluation is that object pairs belonging to the same cluster in a good partition $\mathcal{C}$ should have higher similarities (or, conversely, lower dissimilarities) than those belonging to different clusters.

Once a clustering solution (i.e., a partition) is available for a given dataset, its corresponding Area Under the Curve for Clustering (AUCC) can be computed with the following procedure:

\begin{enumerate}
  \item From the original dataset, compute a similarity matrix of the objects.
  \item Obtain two arrays, which indicate, for each pair of objects, their pairwise:
  \begin{enumerate}
    \item Similarity: readily available from the similarity matrix;
    \item Clustering: 1 if the pair is in the same cluster; 0 otherwise.
    
  \end{enumerate}
  \item Provide the two arrays as input to a standard ROC Analysis procedure in order to obtain the corresponding AUC of the clustering solution, which in this particular context we refer to as Area Under the Curve for Clustering (AUCC). The relation to the usual ROC Analysis in classification is straightforward: similarity values correspond to ``classification thresholds'' whereas pairwise clustering memberships (i.e., binary labels) correspond to the ``true classes''.
\end{enumerate}

The toy example in  Figure~\ref{toy} exemplifies the whole process.

\begin{figure}[!h]

\begin{minipage}{0.6\linewidth} 
\flushleft
\subfloat[Data points and partition with two clusters]{\hspace{0.1cm}\includegraphics[scale=0.35,valign=c]{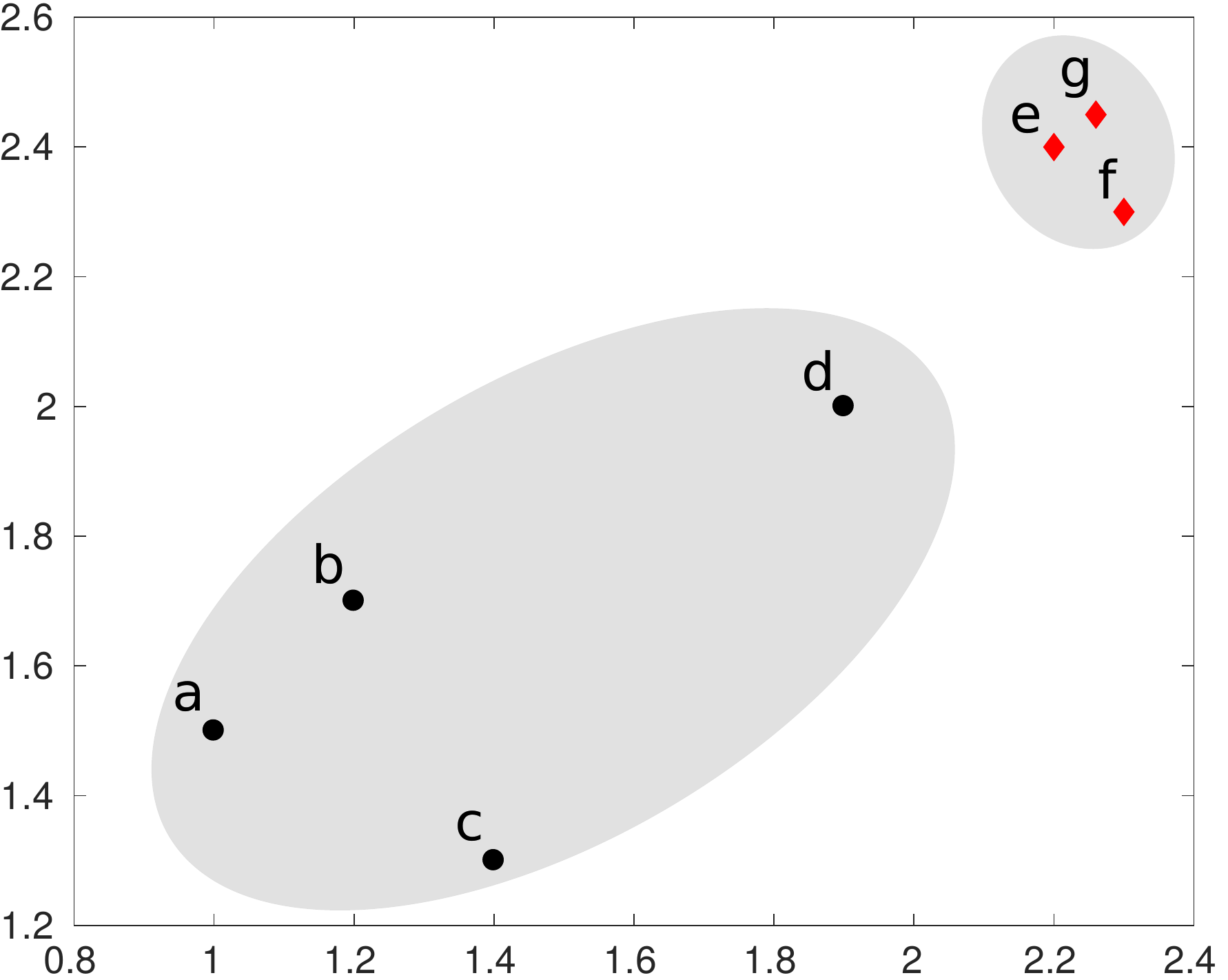}}
\vspace{0.2cm}

\flushleft
\subfloat[Similarity matrix]{ \setlength{\tabcolsep}{4pt}\renewcommand{\arraystretch}{1.2}
\small
\begin{tabular}{c|ccccccc}
&\textbf{a}&\textbf{b}&\textbf{c}&\textbf{d}&\textbf{e}&\textbf{f}&\textbf{g}\\\hline
\textbf{a}&1.00 & 0.82 & 0.72 & 0.35 & 0.05 & 0.03 & 0.00      \\
\textbf{b}&0.82 & 1.00 & 0.72 & 0.52 & 0.23 & 0.20 & 0.18 \\
\textbf{c}&0.72 & 0.72 & 1.00 & 0.45 & 0.14 & 0.15 & 0.09 \\
\textbf{d}&0.35 & 0.52 & 0.45 & 1.00 & 0.68 & 0.68 & 0.63 \\
\textbf{e}&0.05 & 0.23 & 0.14 & 0.68 & 1.00 & 0.91 & 0.95 \\
\textbf{f}&0.03 & 0.20 & 0.15 & 0.68 & 0.91 & 1.00 & 0.90 \\
\textbf{g}&0.00      & 0.18 & 0.09 & 0.63 & 0.95 & 0.90 & 1.00
\end{tabular}}  
\end{minipage}
\begin{minipage}{0.25\linewidth}\small
\renewcommand{\arraystretch}{0.8}\setlength{\tabcolsep}{5pt}
\subfloat[Arrays of pairwise clustering and  similarity for all object pairs]{
\begin{tabular}{ccc}
\toprule
&\multicolumn{2}{c}{Pairwise}\\\cmidrule{2-3}
Pair & Clustering & Similarity                   \\\midrule
\textbf{ab} & 1 & 0.82               \\
\textbf{ac} & 1 & 0.72  \\
\textbf{ad} & 1 & 0.35  \\
\textbf{ae} & 0 & 0.05 \\
\textbf{af} & 0 & 0.03 \\
\textbf{ag} & 0 & 0.00                  \\
\textbf{bc} & 1 & 0.72  \\
\textbf{bd} & 1 & 0.52  \\
\textbf{be} & 0 & 0.23  \\
\textbf{bf} & 0 & 0.20  \\
\textbf{bg}& 0 & 0.18  \\
\textbf{cd} & 1 & 0.45  \\
\textbf{ce} & 0 & 0.14  \\
\textbf{cf} & 0 & 0.15  \\
\textbf{cg} & 0 & 0.09 \\
\textbf{de} & 0 & 0.68  \\
\textbf{df} & 0 & 0.68  \\
\textbf{dg} & 0 & 0.63  \\
\textbf{ef} & 1 & 0.91  \\
\textbf{eg} & 1 & 0.95  \\
\textbf{fg} & 1 & 0.90\\\bottomrule
\end{tabular}    }
\end{minipage}  
\caption{Illustrative example of the Area Under the Curve for Clustering (AUCC) procedure: (a) toy dataset with an arbitrary clustering solution, in which clusters are indicated by a combination of colors and shapes (red diamonds / black circles); (b) similarity matrix between the data objects of the dataset; (c) objects are considered in a pairwise fashion and each pair is associated with the corresponding similarity value and cluster assignment (1 if the pair belongs to the same cluster, 0 otherwise). These pairwise representations can be provided as input to a standard ROC Analysis procedure, resulting in an AUC of 0.9167. This is the AUCC assessment of candidate solution (a).}
\label{toy}
\end{figure}

One ought to note that, in the context of supervised classification, a solution (prediction) is given as real-valued classification scores, while the actual class labels (target result) are represented by binary class labels. In our setup, a clustering solution with any number of clusters is represented as a binary pairwise clustering array, whereas the referential target is represented by real-valued pairwise (dis)similarities intrinsic to the data. Moreover, in the case of clustering, we deal with pairs of objects, as opposed to single objects considered in the traditional classification scenario.

Although the whole validation procedure is described in terms of (dis)similarities, it is important to note that it is not tied to any particular  measure. The only requirements are that: (i)~the (dis)similarity employed in the validation procedure must be the very same (or equivalent) to the one employed during the clustering phase and; (ii) the measure must satisfy the symmetry, positivity and identity properties. Each measure captures a different aspect of the data and any specific choice will depend on the application scenario in hand~\citep{JaiDub88,JasCamCos12,JasCamCos14}. Yet, regardless of the proximity measure in use, the AUCC validation index captures the same essence, that is, it favors partitions in which objects in the same cluster are more similar than objects from different clusters.

\subsection{Equivalence Between AUCC and Baker \& Hubert's Gamma}\label{rocgammaequi}

In this section we discuss the equivalence between the AUCC of a clustering result and its evaluation with the Gamma Index, which is a relative validity criterion introduced by \citet{BakHub75}, based on the Goodman-Kruskal correlation coefficient~\citep{GooKru54}. We initially show that AUCC and Gamma are equivalent to a linear transformation of one another when there are no ties in proximity values (other than self-proximity values).\footnote{This result was originally and preliminarily described in \citep{JaskowiakPhD2015}. An equivalent result, involving the relation between AUC and the 1954 Goodman-Kruskal's rank correlation, was recently rediscovered by \cite{Higham2019} in an unrelated context, involving measures of resolution in meta-cognitive studies.} We then show that the original Gamma Index can be extended in an intuitive way to account for scenarios in which ties may exist, while preserving both the exact relation with AUCC as well as its expected value under a null hypothesis of random clustering solutions. The theoretical expected values for both Gamma and AUCC are also derived in this section as part of our contributions.

Before we proceed, let us recall the definition of the Gamma Index, which can be written as:

\begin{equation} \label{eq:gamma}
\gamma = \frac{s_+ -s_-}{s_+ +s_-},
\end{equation}

\noindent or, equivalently, $1 - \frac{2s_-}{s_{total}}$, with $s_{total} = s_+ -s_-$ and:

\begin{align}
s_+ &= \frac{1}{2}\sum_{l=1}^k \sum_{\substack{\mathbf{x}_i,\mathbf{x}_j \in C_l \\ \mathbf{x}_i \neq \mathbf{x}_j}} \frac{1}{2}\sum_{m=1}^k \sum_{\substack{\mathbf{x}_r \in C_m \\ \mathbf{x}_s \not\in C_m}} \delta ( ||\mathbf{x}_i - \mathbf{x}_j|| < ||\mathbf{x}_r - \mathbf{x}_s ||), \label{s+}\\
s_- &= \frac{1}{2}\sum_{l=1}^k \sum_{\substack{\mathbf{x}_i,\mathbf{x}_j \in C_l \\ \mathbf{x}_i \neq \mathbf{x}_j}} \frac{1}{2}\sum_{m=1}^k \sum_{\substack{\mathbf{x}_r \in C_m \\ \mathbf{x}_s \not\in C_m}} \delta ( ||\mathbf{x}_i - \mathbf{x}_j|| > ||\mathbf{x}_r - \mathbf{x}_s ||),
\end{align}

\noindent where $\delta(.)$ is equal to $1$ if the inequality is satisfied, $0$ otherwise. In the equation above $s_+$ ($s_-$) is the count of occurrences of \emph{object pairs} from the same cluster that have a smaller (greater) dissimilarity $|| \cdot ||$ than that of \emph{object pairs} that belong to different clusters. Intuitively, $s_+$ is expected to account for well placed pairs of objects, whereas $s_-$ should account for misplaced pairs of objects.

It is important to note that the formulation of the Gamma Index as presented above is computationally very expensive, turning out to be prohibitive in most practical applications of cluster analysis. Specifically, it has complexity $O({n^4}/{k})$, where $n$ is the number of data objects and $k$ is the number of clusters in the candidate solution under evaluation~\citep{VenCamHru10}\footnote{Assuming that (a) all dissimilarities $||\cdot||$ are given in advance (otherwise an additional dissimilarity cost would be required --- $O(n^2d)$ in case of Euclidean distance, where $d$ is the dimension of the data space), and (b) cluster sizes are balanced (all proportional to $n/k$, possibly differing by a constant factor)~\citep{VenCamHru10}.
}.
Theorem~\ref{theorem:aucgamma} describes the relation between the outcomes of the evaluation of such a candidate clustering solution by both the Gamma Index as well as the AUCC, whereby the Gamma Index can be computed with a significantly lower computational cost \citep{JaskowiakPhD2015}:

\begin{theorem}~\label{theorem:aucgamma}
Assume that there are no ties in proximity values (except for self-proximity values), i.e., there aren't two pairs of distinct data objects whose (dis)similarity values are exactly the same. Then, the Area Under the ROC Curve for Clustering (AUCC) obtained from the evaluation of a clustering result is equal to $(1 + \gamma) / 2$, where $\gamma$ is the value from the evaluation of the same clustering result with the Gamma criterion from~\citet{BakHub75}, given by Equation~\eqref{eq:gamma}.
\end{theorem}

\begin{proof}
First, let us consider a binary supervised classification problem and a scoring classifier, that is, a classifier that outputs a real-valued score for any data object given as input. Although classification scores may not be interpreted as strict probabilities, the higher their value the higher is the expectancy that the corresponding object should belong to the reference/target class. Scores can thus be associated with a threshold in order to deem objects as negative or positive depending on whether they are below the threshold or not, respectively. Scores can also be used to derive a ranking of the objects: starting from the highest score value, one can rank the objects from~1~(associated with the highest possible rank/score) to a maximum integer associated with the lowest possible rank/score (equal to the number of objects, if there are no ties). Now let us consider the random selection of one positive and one negative object (with respect to their actual class labels).~In this case, the AUC obtained with the evaluation of such a classifier has the interesting statistical interpretation of being equivalent to the probability that it will rank the randomly selected positive object \emph{higher} than the randomly selected negative one~\citep{Fawcett06}.

In the context of clustering validation, each ``object'' of the evaluation corresponds, in fact, to a pair of data objects from the clustering result. The positive and negative classes indicate whether (1) or not (0) a pair of objects belongs to the same cluster, respectively. Finally, scores readily translate into similarity values between pairs of objects.

Recall from the definition of Gamma that terms $s_+$ and $s_-$ are equal to the \emph{number} of occurrences of positive (1) pairs from the cluster solution having a higher ($s_+$) or a lower ($s_-$) similarity value than negative (0) pairs, respectively. Provided that there are no ties in similarity values, candidate occurrences will be counted either to $s_+$ or to $s_-$ (there is no alternative outcome), so the total number of possible counts is \mbox{$s_{total} = s_+ + s_-$,} which depends exclusively on the dataset size ($n$) as well as on the number of clusters ($k$) and their (im)balance (i.e., relative sizes) in the partition under evaluation. Notice that we can define the empirical probabilities with relative frequencies of $s_+$ and $s_-$ by dividing these values by $s_{total}$. Let us denote such empirical probabilities as $P(s_+)$ and $P(s_-)$. Given that such values are all obtained by dividing $s_+$ and $s_-$ by a constant value ($s_{total}$), we can rewrite Gamma as:
\begin{align*}
\gamma &= \frac{P(s_+) - P(s_-)}{P(s_+) + P(s_-)}.\\
\intertext{Since $P(s_+) + P(s_-) = 1$, we have:}
\gamma &= P(s_+) - P(s_-),\\
&= P(s_+) - (1 - P(s_+)),\\
&= P(s_+) - 1 + P(s_+),\\
&= 2P(s_+) - 1.\\
\end{align*}

\noindent Notice that $P(s_+)$ is the empirical probability of ranking a positive example (i.e., a pair of data objects belonging to the same cluster) higher than a negative one, which is exactly the same estimate as the Area Under the ROC Curve value~\citep{Fawcett06}. Hence, $\text{AUCC} = P(s_+) = (\gamma + 1) / 2$. 
\end{proof}

As a side note, given that $\text{AUC} = (\text{Gini} + 1)/2$ \citep{Hand01,Fawcett06}, then the value obtained with the application of Gamma is the very same one obtained with the application of the Gini Coefficient~\citep{Gin12,CerVer12}. To the best of our knowledge, the relation between Baker-Hubert's Gamma and Gini had not been established elsewhere before. Similarly, notice that there is a known relation between AUC and the classic Wilcoxon-Mann-Whitney \emph{U-statistic}: the U-statistic can be defined as a count of the number of times that observations from one sample (say, class positive) are ranked higher than observations from the other sample (say, class negative) according to their scores; by averaging this quantity over all possible pairwise comparisons, the result can be shown to be exactly equivalent to the AUC \citep{Hanley1982,Mason2002} and, by transitivity using Theorem~\ref{theorem:aucgamma}, also equivalent to Gamma.

{\bf Computational Complexity:} As previously mentioned, the original formulation of Gamma~\citep{BakHub75} is computationally prohibitive for most real-world applications, due to its $O({n^4}/{k})$ complexity~\citep{VenCamHru10}. For instance, this has prevented its evaluation in datasets as small as $500$ objects in the experimental study performed by~\citet{VenCamHru10}. As pointed out by \citet{Fawcett06}, computing the AUC for a binary classification problem with $n$ objects has~$O(n \log n)$ complexity. Note that in the case of clustering evaluation we are dealing with pairs of objects, thus we have an $O(n^2 \log n)$ time complexity\footnote{Apart from the cost to obtain the dissimilarity matrix, $\mathbf{D}$, which is also required by Gamma.} for the Area Under the Curve for Clustering (AUCC), a considerable reduction when compared to the original Gamma.

\subsection{Expected Value Property}\label{expected_val_prop}

In this section we show that, given a dataset $\mathbf{X} = \{\mathbf{x}_1,\dots,\mathbf{x}_n\}$ of finite size $n$ and a dissimilarity measure $||\mathbf{x}_i - \mathbf{x}_j||$ associated with each pair of data objects $\mathbf{x}_i$ and $\mathbf{x}_j$, the expected value of the Gamma Index under a null distribution of random clustering solutions is zero and, accordingly, the expected value of AUCC is 0.5. This results holds true for any given value of $k$, i.e., it is valid irrespective of the number of clusters assumed in the null model. It also holds true independently of the dataset size $n$ and, as we will show, irrespective of relative cluster sizes, i.e., cluster (im)balance. This is a desirable property for two reasons: (a) it allows for a better interpretation of AUCC values, that is, how far/close a given candidate clustering solution is from random; and, more importantly, (b) it ensures that the use of AUCC as an internal/relative validity criterion~is not biased by the number or (im)balance of clusters in the partitions being compared.

Before we formalize these results, it is paramount to stress that the null model is defined at the individual data object level, as a random assignment of objects to clusters; it is \emph{not} directly defined in terms of the binary relation $\mathcal{C}^p$ on which AUCC relies because, by randomly assigning binary labels to pairs of objects, we would actually account for in our calculations encodings that do not correspond to any feasible partition of the data. For instance, consider a dataset $\{\mathbf{x}_1$, $\mathbf{x}_2$, $\mathbf{x}_3\}$, case in which there are three possible pairs, $(\mathbf{x}_1,\mathbf{x}_2)$, $(\mathbf{x}_1,\mathbf{x}_3)$, $(\mathbf{x}_2,\mathbf{x}_3)$. A hypothetical random encoding $[1~0~1]$ says that $\mathbf{x}_1$ and $\mathbf{x}_2$ are in the same cluster, and so do $\mathbf{x}_2$ and $\mathbf{x}_3$, whereas $\mathbf{x}_1$ and $\mathbf{x}_3$ are in different clusters, which is impossible. This prevents us from directly evoking the well-known property that the expected AUC for chance is 0.5, because this result assumes that each and every element has, independently of other elements, the same fixed probabilities of being assigned to each class. The example above shows that the elements of the binary relations $\mathcal{C}^p$ do not satisfy this assumption as they are not independent. Note that independence is also a critical assumption behind the use of the \emph{U-statistic}, which is equivalent to the AUC \citep{Mason2002}. Our next result circumvents this hurdle by working at the object (rather than pairwise object) level:

\begin{theorem}~\label{theorem:Gammazero}
Assuming a null model in which every clustering solution with $k$ clusters (as a valid partition of $n$ objects) is equally likely, the expected value of the Gamma Index is zero ($\gamma = 0$).
\end{theorem}
\begin{corollary}~\label{corollary:aucchalf}
Assuming a null model in which every clustering solution with $k$ clusters (as a valid partition of $n$ objects) is equally likely, the expected value of AUCC is $0.5$.
\end{corollary}
\begin{proof}
For the sake of simplicity and without loss of generality, we initially assume here that there are no ties in the dissimilarities between pairs of objects (the more general case involving ties will be discussed in Section \ref{ties}). Since there are no ties, the quantity \mbox{$s_{total} = s_+ + s_-$} depends exclusively on the dataset size as well as on the number and the (im)balance of clusters in the partition under evaluation. For a given dataset of size $n$ and a fixed number of clusters $k$ of interest, $s_{total}$ depends solely on the relative cluster sizes. Once again, for the sake of simplicity and without loss of generality, we will \emph{initially} assume that the relative cluster sizes are also fixed in the null model. This may not be unreasonable in practice since one may want to compare a given candidate clustering solution against a null model of random solutions of exactly the same nature. In spite of that, we will subsequently show that the expected value actually doesn't change if we generalize/extend the null model such that the expectation is computed across random partitions with all possible cluster size proportions. In addition, since the result holds irrespective of the number of clusters, then it can be trivially shown to hold for an even more general null model in which the expected value for chance is computed over all possible partitions of the data, regardless of $k$ and/or (im)balance. 

For given $n$, $k$, and cluster size proportions, $s_{total}$ is a constant and, hence, the expected value of $\gamma$ can be derived from Equation~(\ref{eq:gamma}) as: 

\begin{equation} \label{eq:gamma_expected1}
E_{{\mathbb C}_k}\{\gamma\} = \frac{E_{{\mathbb C}_k}\{s_+\} -E_{{\mathbb C}_k}\{s_-\}}{s_{total}},
\end{equation}

\noindent where the expectation is taken evenly across the set ${{\mathbb C}_k}$ of all possible valid partitions \mbox{$\mathcal{C}_k = \{C_1,\dots,C_k\}$} consisting of $k$ clusters of fixed relative sizes $|C_1|/n, |C_2|/n, \cdots, |C_k|/n$ ($|\cdot|$ stands for set cardinality) or any permutation of these. In order to compute $E_{{\mathbb C}_k}\{s_+\}$ (and, subsequently, $E_{{\mathbb C}_k}\{s_-\}$
in an analogous fashion) it is worth noticing that term $s_+$ can be written in a completely equivalent form as:

\begin{equation}~\label{alternative_s+}
\begin{array}{lcl}
s_+ &  =  & \displaystyle  \sum_{\mathbf{x}_i \in \mathbf{X}} \; \; \sum_{\substack{\mathbf{x}_j \in \mathbf{X} \\ j \neq i}} \; \; \sum_{\substack{\mathbf{x}_s \in \mathbf{X} \\ s \neq j \neq i}} \Biggl [ \delta ( ||\mathbf{x}_i - \mathbf{x}_j|| < ||\mathbf{x}_i - \mathbf{x}_s ||) \cdot \mu_{\mathcal{C}_k}(i,j,s)\Biggr ] + \\
& & \displaystyle + \sum_{\mathbf{x}_i \in \mathbf{X}} \; \; \; \sum_{\substack{\mathbf{x}_j \in \mathbf{X} \\ j \neq i}} \; \; \; \sum_{\substack{\mathbf{x}_r \in \mathbf{X} \\ r \neq j \neq i}} \; \; \sum_{\substack{\mathbf{x}_s \in \mathbf{X} \\ s \neq r \neq j \neq i}} \Biggl [ \delta ( ||\mathbf{x}_i - \mathbf{x}_j|| < ||\mathbf{x}_r - \mathbf{x}_s ||) \cdot \phi_{\mathcal{C}_k}(i,j,r,s)\Biggr ],
\end{array}
\end{equation}

\noindent where $\mu_{\mathcal{C}_k}(i,j,s)$ is an indicator function that takes as argument the indexes $i \neq j \neq s$ of three different objects of the dataset (a triple with no duplicates) and returns 1 if and only if the first two objects belong to the same cluster ($\mathbf{x}_i,\mathbf{x}_j \in C_l$) whereas the third object belongs to a different cluster ($\mathbf{x}_s \in C_m$, $m \neq l$) in partition $\mathcal{C}_k$ ($\{C_l, C_m\} \subset \mathcal{C}_k$); otherwise, $\mu_{\mathcal{C}_k}(i,j,s)$ is equal to zero. Similarly, $\phi_{\mathcal{C}_k}(i,j,r,s)$ is an indicator function that takes as argument the indexes $i \neq j \neq r \neq s$ of four different objects of the dataset (a quadruple with no duplicates) and returns 1 if and only if the first two objects belong to the same cluster ($\mathbf{x}_i,\mathbf{x}_j \in C_l$) whereas the other two objects belong to separate clusters ($\mathbf{x}_r \in C_m, \mathbf{x}_s \not\in C_m$);\footnote{Note that $C_m$ is not necessarily different from $C_l$, they may or may not be the same cluster in partition $\mathcal{C}_k$.} otherwise, $\phi_{\mathcal{C}_k}(i,j,r,s)$ is equal to zero.

The main advantage of the above representation is that, unlike the previous equivalent definition of $s_+$ in Equation~(\ref{s+}), the summation indexes in Equation~(\ref{alternative_s+}) do \emph{not} depend on the clustering solution $\mathcal{C}_k$. Obviously, the summations are now covering an augmented set of terms, namely, terms involving the comparison of pairwise distances from all triples or quadruples of distinct objects in the dataset. The additional/augmented terms (and only those) are, however, cancelled out by a null value of the respective indicator function, namely, $\mu_{\mathcal{C}_k}(\cdot)$ for triples and $\phi_{\mathcal{C}_k}(\cdot)$ for quadruples.

It is worth noticing that functions $\mu_{\mathcal{C}_k}(\cdot)$ and $\phi_{\mathcal{C}_k}(\cdot)$ depend only on the partition $\mathcal{C}_k$ under evaluation, they do not depend on the pairwise dissimilarities between data objects. Conversely, function $\delta(\cdot)$ depends \emph{only} on the pairwise dissimilarities, it does \emph{not} depend on any partition of the data. From this observation, we can write the expectation $E_{{\mathbb C}_k}\{s_+\}$ from Equation~(\ref{alternative_s+}) as:

\begin{equation}~\label{expect_s+}
\begin{array}{lcl}
E_{{\mathbb C}_k}\{s_+\} &  =  & \displaystyle  \sum_{\mathbf{x}_i \in \mathbf{X}} \; \; \sum_{\substack{\mathbf{x}_j \in \mathbf{X} \\ j \neq i}} \; \; \sum_{\substack{\mathbf{x}_s \in \mathbf{X} \\ s \neq j \neq i}} \Biggl [ \delta ( ||\mathbf{x}_i - \mathbf{x}_j|| < ||\mathbf{x}_i - \mathbf{x}_s ||) \cdot E_{{\mathbb C}_k}\{\mu_{\mathcal{C}_k}(i,j,s)\}\Biggr ] + \\
& & \displaystyle + \sum_{\mathbf{x}_i \in \mathbf{X}} \; \; \; \sum_{\substack{\mathbf{x}_j \in \mathbf{X} \\ j \neq i}} \; \; \; \sum_{\substack{\mathbf{x}_r \in \mathbf{X} \\ r \neq j \neq i}} \; \; \sum_{\substack{\mathbf{x}_s \in \mathbf{X} \\ s \neq r \neq j \neq i}} \Biggl [ \delta ( ||\mathbf{x}_i - \mathbf{x}_j|| < ||\mathbf{x}_r - \mathbf{x}_s ||) \cdot E_{{\mathbb C}_k}\{\phi_{\mathcal{C}_k}(i,j,r,s)\}\Biggr ].
\end{array}
\end{equation}

Notice that terms $E_{{\mathbb C}_k}\{\mu_{\mathcal{C}_k}(i,j,s)\}$ and $E_{{\mathbb C}_k}\{\phi_{\mathcal{C}_k}(i,j,r,s)\}$ can be readily interpreted as the fraction of all partitions $\mathcal{C}_k \in {{\mathbb C}_k}$ (i.e., the fraction of the population of valid partitions comprised by the null model, \mbox{${{\mathbb C}_k}$}) such that the corresponding indicator functions return a non-zero (unit) value. Since the indicator functions $\mu_{\mathcal{C}_k}(i,j,s)$ and $\phi_{\mathcal{C}_k}(i,j,r,s)$ do \emph{not} depend on any intrinsic property of data objects, they depend instead only on the \emph{cluster labels} imposed to those specific objects indexed by the functions' arguments, the \emph{expected values} $E_{{\mathbb C}_k}\{\mu_{\mathcal{C}_k}(i,j,s)\}$ and $E_{{\mathbb C}_k}\{\phi_{\mathcal{C}_k}(i,j,r,s)\}$ will be the same irrespective of the indexes $i,j,r,s$. In other words, if we fix any three (four) distinct objects and average $\mu_{\mathcal{C}_k}(i,j,s)$ ($\phi_{\mathcal{C}_k}(i,j,r,s)$) over all partitions in ${\mathbb C}_k$, across which only the cluster labels of objects are permuted, then the result will be the same (a constant). We will call these constants $E_{{\mathbb C}_k}\{\mu_{\mathcal{C}_k}\}$ and $E_{{\mathbb C}_k}\{\phi_{\mathcal{C}_k}\}$ for short, whereby we can rewrite Equation~(\ref{expect_s+}) as:  

\begin{equation}~\label{expect_s+2}
\begin{array}{lcl}
E_{{\mathbb C}_k}\{s_+\} &  =  & \displaystyle E_{{\mathbb C}_k}\{\mu_{\mathcal{C}_k}\} \cdot  \sum_{\mathbf{x}_i \in \mathbf{X}} \; \; \sum_{\substack{\mathbf{x}_j \in \mathbf{X} \\ j \neq i}} \; \; \sum_{\substack{\mathbf{x}_s \in \mathbf{X} \\ s \neq j \neq i}} \delta ( ||\mathbf{x}_i - \mathbf{x}_j|| < ||\mathbf{x}_i - \mathbf{x}_s ||) \; \; + \\
& & \displaystyle E_{{\mathbb C}_k}\{\phi_{\mathcal{C}_k}\} \cdot  \sum_{\mathbf{x}_i \in \mathbf{X}} \; \; \; \sum_{\substack{\mathbf{x}_j \in \mathbf{X} \\ j \neq i}} \; \; \; \sum_{\substack{\mathbf{x}_r \in \mathbf{X} \\ r \neq j \neq i}} \; \; \sum_{\substack{\mathbf{x}_s \in \mathbf{X} \\ s \neq r \neq j \neq i}} \delta ( ||\mathbf{x}_i - \mathbf{x}_j|| < ||\mathbf{x}_r - \mathbf{x}_s ||).
\end{array}
\end{equation}

Following an analogous reasoning, we can also write $E_{{\mathbb C}_k}\{s_-\}$ as:

\begin{equation}~\label{expect_s-}
\begin{array}{lcl}
E_{{\mathbb C}_k}\{s_-\} &  =  & \displaystyle E_{{\mathbb C}_k}\{\mu_{\mathcal{C}_k}\} \cdot  \sum_{\mathbf{x}_i \in \mathbf{X}} \; \; \sum_{\substack{\mathbf{x}_j \in \mathbf{X} \\ j \neq i}} \; \; \sum_{\substack{\mathbf{x}_s \in \mathbf{X} \\ s \neq j \neq i}} \delta ( ||\mathbf{x}_i - \mathbf{x}_j|| > ||\mathbf{x}_i - \mathbf{x}_s ||)  \; \; + \\
& & \displaystyle E_{{\mathbb C}_k}\{\phi_{\mathcal{C}_k}\} \cdot  \sum_{\mathbf{x}_i \in \mathbf{X}} \; \; \; \sum_{\substack{\mathbf{x}_j \in \mathbf{X} \\ j \neq i}} \; \; \; \sum_{\substack{\mathbf{x}_r \in \mathbf{X} \\ r \neq j \neq i}} \; \; \sum_{\substack{\mathbf{x}_s \in \mathbf{X} \\ s \neq r \neq j \neq i}} \delta ( ||\mathbf{x}_i - \mathbf{x}_j|| > ||\mathbf{x}_r - \mathbf{x}_s ||).
\end{array}
\end{equation}

Now, notice that, for every triple $(i,j,s)$ (i.e., $ \forall \; i \neq j \neq s$) such that $\delta ( ||\mathbf{x}_i - \mathbf{x}_j|| < ||\mathbf{x}_i - \mathbf{x}_s ||) = 1$ in Equation~(\ref{expect_s+2}), there is a triple $(i,s,j)$ for which $\delta ( ||\mathbf{x}_i - \mathbf{x}_s|| > ||\mathbf{x}_i - \mathbf{x}_j ||) = 1$ in Equation~(\ref{expect_s-}), and vice versa. Similarly, for every quadruple $(i,j,r,s)$ (i.e., $ \forall \; i \neq j \neq r \neq s$) such that \mbox{$\delta ( ||\mathbf{x}_i - \mathbf{x}_j|| < ||\mathbf{x}_r - \mathbf{x}_s ||) = 1$} in Equation~(\ref{expect_s+2}), there is a quadruple $(r,s,i,j)$ for which $\delta ( ||\mathbf{x}_r - \mathbf{x}_s|| > ||\mathbf{x}_i - \mathbf{x}_j ||) = 1$ in Equation~(\ref{expect_s-}), and vice versa.

Therefore, $E_{{\mathbb C}_k}\{s_+\} = E_{{\mathbb C}_k}\{s_-\}$ and, from Equation~(\ref{eq:gamma_expected1}), it follows that $E_{{\mathbb C}_k}\{\gamma\} = 0$, i.e., the expected value of the Gamma Index under the assumed null model of random clustering solutions is zero. Finally, from this result and using Theorem~\ref{theorem:aucgamma}, $E_{{\mathbb C}_k}\{AUCC\} = 0.5$ follows straightforwardly. 

{\bf Extended Null Model:} The above results prove both the theorem and the corollary. However, the proof assumes that the relative sizes in the $k$ clusters contained in any random solution of the null model are fixed. If this condition is not satisfied, $s_{total} = s_+ + s_-$ will no longer be a constant (it will vary across different partitions in the null model), case in which Equation~(\ref{eq:gamma_expected1}) does not hold true, at least not simultaneously across the entire population of random partitions ${\mathbb C}_k$.

We can extend the above results to cases in which a more general null model is considered, where ${\mathbb C}_k$ contains partitions with any and all possible cluster size proportions, rather than a prefixed one. This can be achieved by noticing that ${\mathbb C}_k$ is a finite set, and all elements in this set (random clustering solutions) are equally likely, case in which the mathematical expectation of a function ($\gamma$) of the elements in this set can be written as the ordinary arithmetic mean of the function evaluations for each element in the set: $ E_{{\mathbb C}_k}\{\gamma\} = \frac{1}{|{\mathbb C}_k|} \sum_{{\mathcal C}_k \in {\mathbb C}_k} \gamma({\mathcal C}_k)$. If we arbitrarily group the set ${\mathbb C}_k$ into any chosen collection of disjoint subsets (of random clustering solutions) ${\mathbb S}_i \subset {\mathbb C}_k$, such that $\bigcup_{i} {\mathbb S}_i = {\mathbb C}_k$, it is trivial to show that the expectation across the entire set can be written as an average of the expectations within each subset ($E_{{\mathbb S}_i}(\gamma)$) weighted by their cardinalities, i.e.: 

\begin{align}
E_{{\mathbb C}_k}\{\gamma\} & = \frac{ \sum_{{\mathbb S}_i \in {\mathbb C}_k} |{\mathbb S}_i| \cdot E_{{\mathbb S}_i}(\gamma)}{\sum_{{\mathbb S}_i \in {\mathbb C}_k} |{\mathbb S}_i|} \nonumber \\ 
& = \frac{ \sum_{{\mathbb S}_i \in {\mathbb C}_k} |{\mathbb S}_i| \cdot E_{{\mathbb S}_i}(\gamma)}{|{\mathbb C}_k|}.
\label{weighted_exp}
\end{align}

Since this result is valid for any arbitrary subdivision of ${\mathbb C}_k$ as described above, for mathematical convenience we chose a subdivision such that every subgroup ${\mathbb S}_i$ contains all and only the clustering solutions in ${\mathbb C}_k$ that share the same cluster size proportions. In other words, each ${\mathbb S}_i$ is associated with a unique (im)balance of clusters, thence $s_{total}$ is constant for all partitions within ${\mathbb S}_i$ and all the results above in this proof are also valid for ${\mathbb S}_i$. Specifically, $E_{{\mathbb S}_i}\{\gamma\} = 0, \; \forall i$. Therefore, it follows from Equation~(\ref{weighted_exp}) that $E_{{\mathbb C}_k}\{\gamma\} = 0$ and, by evoking Theorem~\ref{theorem:aucgamma} we have $E_{{\mathbb C}_k}\{AUCC\} = 0.5$. 
\end{proof}


\subsection{Ties in (Dis)similarity Values}\label{ties}

In principle, the relation between the Gamma Index and AUCC established in Theorem~\ref{theorem:aucgamma} assumes that there are no ties in the real-valued thresholds used to compute the area under the ROC curve. To better understand what happens when ties are present, let's consider the toy example in Table~\ref{tab:ties_example}. In a classification assessment scenario, the six instances would be data objects with binary class labels and a classification score associated with the positive class, whereas in a clustering assessment scenario they would correspond to pairs of objects with binary clustering assignment relations (``labels'') and a pairwise similarity value (``score'').

\begin{table}[!h]
    \centering
    \caption{Illustrative example of a classification or clustering problem involving ties.}
    \begin{tabular}{ccc}
    \toprule
    Instance & Label & Score\\
    \midrule
1 & 1 & 0.75 \\  
2 & 0 & 0.50 \\ 
3 & 1 & 0.50 \\ 
4 & 1 & 0.50 \\
5 & 0 & 0.25 \\ 
6 & 0 & 0.20 \\ 
\bottomrule
\end{tabular}
\label{tab:ties_example}
\end{table}

Notice that instances 2, 3 and 4 share exactly the same score of 0.5, i.e., they are tied. As the real-valued threshold used to compute the AUC moves from 0.25 to 0.75 (or vice-versa), it is not clear how exactly the ROC curve should move from point $(FPR,TPR) = \left(\frac{1}{3},1\right)$ to point $(FPR,TPR) = \left(0,\frac{1}{3}\right)$ in the ROC graph, respectively. This is because the ROC curve depends on the relative rank/order of the instances according to their scores, however, the relative order among instances sharing the same score cannot be uniquely determined. Each rank permutation of those instances would incur a different area under the curve. Notice that the uncertainty around the final, total area comes exclusively from the subarea of the unit square comprising the rectangle with diagonal/opposite vertices $\left(\frac{1}{3},1\right)$ and $\left(0,\frac{1}{3}\right)$. The area of this rectangle is $\frac{1}{3} \times \frac{2}{3} = \frac{2}{9}$.

The most accepted and widely adopted approach to compute the ROC curve in case of ties is so-called ``walk along the diagonal'', which in our pedagogic example in Table~\ref{tab:ties_example} corresponds to connecting the points $\left(\frac{1}{3},1\right)$ and $\left(0,\frac{1}{3}\right)$ with a straight line \citep{Fawcett06}. In terms of the area under the curve, this corresponds to assigning to the final area precisely half of the subarea subject to uncertainty due to ties \big(i.e., $\frac{1}{2} \times \frac{2}{9} = \frac{1}{9}$\big); in the above example, the final area will thus amount to a total of $0.8888$. From the probabilistic interpretation of ROC curves previously discussed in Sections~\ref{intro} and \ref{rocgammaequi}, this approach corresponds to assigning half of the fraction of probability involved in ties (and whose assignment is unclear) to the computed area under the curve as an estimate of the probability that positive (1) instances will be ranked higher than negative (0) instances. Accordingly, the other half will be assigned to this probability's complement, i.e., the estimated probability that negative instances will be ranked higher.

In the clustering assessment scenario, the ``walk along the diagonal'' approach described above corresponds to assigning half of the fraction of probability involved in ties to the computed area under the curve as an estimate of the probability that pairs of objects belonging to the same cluster (1) will be ranked higher than pairs belonging to different clusters (0). Let's call the fraction of probability (i.e. the subarea) involved in ties as $P_t$, such that $P_t = \frac{2}{9}$ in our example. When computing the Gamma Index, this is the probability associated with the outcomes that are \emph{not} accounted for by terms $s_+$ and $s_-$, namely, those outcomes involving similarity ties. In the presence of ties, the quantity $s_{total} = s_+ \, + \, s_-$ is \emph{no longer} equal to a constant that represents the total number of possible counts and depends exclusively on $n$, $k$ and relative cluster sizes. Rather, such a constant (renamed \mbox{$s'_{total}$} hereafter) is now equal to \mbox{$s'_{total} = s_+ + s_- + s_0$}, where $s_0$ is given by:

\begin{align*}
s_0 &= \frac{1}{2}\sum_{l=1}^k \sum_{\substack{\mathbf{x}_i,\mathbf{x}_j \in C_l \\ \mathbf{x}_i \neq \mathbf{x}_j}} \frac{1}{2}\sum_{m=1}^k \sum_{\substack{\mathbf{x}_r \in C_m \\ \mathbf{x}_s \not\in C_m}} \delta ( ||\mathbf{x}_i - \mathbf{x}_j|| = ||\mathbf{x}_r - \mathbf{x}_s ||), \label{s0}
\end{align*}

\noindent such that $P_t = s_0 / s'_{total}$. When there are ties, $s_0 \neq 0$, $P_t \neq 0$ and, because $P(s_+) + P(s_-) + P_t = 1$, it follows that $P(s_+) + P(s_-) \neq 1$. In this case, the relation between the Gamma Index and AUCC established in Theorem~\ref{theorem:aucgamma} is no longer valid. As a simple proof by counter-example, in the dataset of Table~\ref{tab:ties_example} it is trivial to show that $s_+ = 7$ and $s_- = 0$, so if Equation~(\ref{eq:gamma}) were to be used, the result would be $\gamma = \frac{7 - 0}{7 + 0} = 1$; evoking Theorem~\ref{theorem:aucgamma} one would in that case get $\text{AUCC} = (1 + \gamma)/2 = 1$, which is in contradiction with the value obtained by computing the area under the ROC curve walking along the diagonal to resolve ties ($\text{AUCC} = 0.8888$). The contradiction is caused by the presence of ties, which violates Theorem~\ref{theorem:aucgamma}'s assumption.

Under the ``walk along the diagonal'' assumption for dealing with ties in the AUCC computation, the relation in Theorem~\ref{theorem:aucgamma} can be reestablished by also distributing half of the probability $P_t$ to $P(s_+)$ and the other half to $P(s_-)$ when computing Gamma. In other words, we must assign an additional $s_0/2$ amount to term $s_+$ and the same $s_0/2$ additional amount to term $s_-$ when computing $\gamma$ in Equation~(\ref{eq:gamma}), which is thereby generalized as:

\begin{equation} \label{eq:gamma_gen}
\gamma = \frac{\left(s_+ + \frac{s_0}{2}\right) - \left(s_- + \frac{s_0}{2}\right)}{\left(s_+ + \frac{s_0}{2}\right) + \left(s_- + \frac{s_0}{2}\right)} = \frac{s_+ - s_-}{s_+ + s_- + s_0},
\end{equation}

 \noindent and clearly reduces back to Equation~(\ref{eq:gamma}) in the absence of ties. By using Equation~(\ref{eq:gamma_gen}) instead of Equation~(\ref{eq:gamma}) in the dataset of Table~\ref{tab:ties_example}, one has $s_0 = 2$, \mbox{$\gamma = \frac{7 - 0}{7 + 0 + 2} = \frac{7}{9}$} and, in this case, $\text{AUCC} = (1 + \gamma)/2 = 0.8888$ follows from Theorem~\ref{theorem:aucgamma} as expected. 

Equation~(\ref{eq:gamma_gen}) allows Theorem~\ref{theorem:aucgamma} to be stated more broadly without any particular assumption involving ties. It is worth noticing that, by distributing the probability/area associated with ties evenly between $s_+$ and $s_-$ as in Equation~(\ref{eq:gamma_gen}) --- or equivalently, ``walking along the diagonal'' when computing AUCC --- we do not change the fact that $E_{{\mathbb C}_k}\{s_+\} = E_{{\mathbb C}_k}\{s_-\}$ and, as a consequence, the expected value properties in Theorem~\ref{theorem:Gammazero} and Corollary~\ref{corollary:aucchalf} remain valid for $\gamma$ and $\text{AUCC}$, respectively.

Finally, it is also worth noticing that at least a couple of (not-so-common) alternative approaches to resolve ties in ROC analysis exist \citep{Fawcett06} that could also be adopted to compute AUCC. In particular, the \emph{optimistic} approach, which fully assigns the total amount of probability/subarea associated with ties to the final area under the curve --- would be equivalent to allocating the corresponding additional amount $s_0$ entirely to $s_+$ (none to $s_-$) when computing the Gamma Index. In contrast, the \emph{pessimistic} approach --- where none of the probability/subarea associated with ties is assigned to the area under the curve --- would be equivalent to allocating the additional amount $s_0$ entirely to $s_-$ when computing Gamma. While these alternative approaches for AUCC computation and the corresponding aforementioned modifications to the Gamma Index would keep the relation in Theorem~\ref{theorem:aucgamma} valid in spite of the presence/absence of ties, these approaches would clearly bias the expected value of either $s_+$ or $s_-$ such that $E_{{\mathbb C}_k}\{s_+\} \neq E_{{\mathbb C}_k}\{s_-\}$. Accordingly, Theorem~\ref{theorem:Gammazero} and Corollary~\ref{corollary:aucchalf} would no longer be valid in the presence of ties. 

\section{Experimental Evaluation}\label{eval}

\subsection{Agreement with External Evaluation}\label{uci}

In order to experimentally assess the use of AUCC in the relative clustering validation scenario, we have employed the same evaluation methodology proposed by~\citet{VenCamHru09,VenCamHru10}. In short, it assumes that the best relative validity criteria should have the highest correlation with an external validity index. External indices are based on comparisons between the clustering results and a ground-truth, which are available for simulated and benchmark datasets. Evaluations of relative criteria based on their correlations with an external index have also been carried out e.g. by~\citep{JasMouFurCam15, Nguyen2020}. The procedure can be summarized as follows:

\begin{enumerate}
\item Given a dataset, generate partitions/solutions with different numbers of clusters ($k$), usually with $2 \le k \le \sqrt{n}$ (configuration we adopt here), employing one or more clustering algorithms;

\item Determine the quality of the partitions w.r.t. to the relative validity criteria under scrutiny;

\item Determine the quality of each partition according to one (or more) external validity criterion;

\item Measure the correlation between the unsupervised and supervised evaluations provided by each of the relative and the external validity criteria, respectively.
\end{enumerate}

To generate a diverse collection of clustering partitions (Step 1 of the evaluation methodology), we employ the well-known k-means clustering algorithm \citep{Mac67} and four variants of Hierarchical Clustering Algorithms (HCAs)~\citep{JaiDub88}, namely,
Single-Linkage, Average-Linkage, Complete-Linkage, and Ward's. For each dataset we generate partitions in the range $k \in \{2, \dots,
k_{max}\}$, with $k_{max} = \left\lceil \sqrt{n} \right\rceil$. In the case of k-means, for each $k$, 100 initializations are undertaken and the partition with best MSE (Mean Squared Error) is then selected for further evaluation. External agreement of partitions with respect to the true labels are obtained with the Adjusted Rand Index (ARI)~\citep{HubAra85,AmiGonArtVer09}. Correlations between relative and external evaluations are given by the Pearson correlation coefficient~\citep{Pea1895}.

To place the results of AUCC into perspective we consider a collection of 28~relative validity criteria commonly employed in the literature as baselines \citep{Nguyen2020,Zhou2021}. These are: Calinski-Harabasz~(VRC) \citep{Calinski74}, Davies–Bouldin (DB)~\citep{DavBou79}, Dunn’s
Index \citep{Dun74}, 17 variants of Dunn’s Index \citep{BezPal98},
PBM \citep{PakBanMau04}, C-Index \citep{HubLev76}, Point-Biserial~\citep{Mil81}, C/Sqrt(k) \citep{RatLan78,Hil80}, Silhouette Width Criterion (SWC) \citep{Rou87}, Simplified Silhouette
Width Criterion (SSWC)~\citep{Hruschka2006}, Alternative Silhouette Width Criterion (ASWC)~\citep{Hruschka2004}, and
Alternative Simplified Silhouette Width Criterion (ASSWC)~\citep{VenCamHru09}. 

We evaluate AUCC alongside the aforementioned baseline validity criteria in 10 real datasets with varied characteristics
in terms of their numbers of objects, dimensions and clusters in the reference ground-truth partition. These are: (i) the Yeast Galactose (Yeast) and Cell Cycle from Yeung et al.~\citep{YeuFraMurRaf01} as well as eight datasets from UCI, as summarized in Table~\ref{tab:uci}. 

\begin{table}[!h]
\scriptsize
    \centering
    \caption{Real datasets employed in experiments.}
    \begin{tabular}{clccc}
    \toprule
    \#&Dataset&\# Objects&\# Dimensions&\# Clusters\\
    \midrule
1&Balance Scale& 625& 4& 3\\
2&Cell Cycle&237& 17& 4\\
3&Control Chart (KDD)& 600& 60 &6\\
4&E. Coli& 336& 7& 8\\
5&Iris &150& 4& 3\\
6&Karhunen&2000& 64 &10\\
7&Sonar& 208& 60& 2\\
8&Vehicle& 846 &18 &4\\
9&Wisconsin Breast Cancer &683& 9& 2\\
10&Yeast &205 &20 &4\\
\bottomrule
\end{tabular}
\label{tab:uci}
\end{table}

Evaluation results are depicted in Table~\ref{tab:uciresults}. Validity criteria are ranked by their average correlation values. Although we evaluated a total of 18 Dunn formulations, for the sake of simplicity, we show the results for the best performer. AUCC 
ranked second best overall, below Point-Biserial (PB) only. Aggregated results across multiple datasets should be taken with a grain of salt though. Different criteria address the multi-faceted problem of clustering evaluation from different angles and may emphasize more or less certain particular aspects. It is well-known that no single criterion should be expected to outperform all the others in all problems. Instead, different criteria are expected to perform better/worse than others in different problems or scenarios. In fact, notice in Table~\ref{tab:uciresults} that, while a subset of criteria have been outperformed by others within the collection of ten datasets involved in our experiments, five different criteria, namely PB, AUCC, C-Index, C/Sqrt(k) and VRC, exhibited the best/top performance in at least one dataset.
For this reason, it is widely accepted that an analyst should not rely on a single criterion for unsupervised clustering evaluation \citep{BezPal98,JasMouFurCam15}; naive attempts to elect a single criterion as the best one overall are inevitably fruitless, unless they focus on specific classes of problems/scenarios.

\begin{table}[!h]
\centering
\caption{Evaluation results of AUCC and baseline relative criteria on 10 real datasets. 
Each cell displays the Pearson correlation between the relative evaluation and the external evaluation obtained with ARI. Top performance for each dataset (in columns ``1'' to ``10'') is highlighted in bold.}
\ssmall
\newcolumntype{C}{>{\centering\arraybackslash}m{0.6cm}<{}}
\newcolumntype{L}{>{\arraybackslash}m{1.2cm}<{}}
\newcolumntype{R}{>{\centering\arraybackslash}m{0.65cm}<{}}
\begin{tabular}{LCCCCCCCCCC|CCC}
\toprule
Dataset \#&1&2&3&4&5&6&7&8&9&10&Best&Avg.&Worst\\
\cmidrule(lr){1-11}\cmidrule(lr){12-14}
PB&0.79&\textBF{0.91}&0.61&\textBF{0.97}&\textBF{0.69}&\textBF{0.89}&0.31&0.40&\textBF{0.98}&0.57&0.98&0.71&0.31\\
AUCC&0.48&0.60&\textBF{0.75}&0.76&0.13&0.84&\textBF{0.70}&0.78&0.91&\textBF{0.77}&0.91&0.67&0.13\\
C-Index&0.53&0.75&\textBF{0.75} &0.83&-0.07&0.88&0.64&0.78&0.81&0.57&0.88&0.65&-0.07\\
C/Sqrt(k)&\textBF{0.88}&0.82&0.10&0.81&0.59&0.76&0.32&0.71&0.73&0.58&0.88&0.63&0.10\\
SWC&0.76&0.84&0.06&0.65&0.34&0.80&0.38&0.82&0.88&0.73&0.88&0.62&0.06\\
ASWC&0.70&0.50&0.19&0.58&0.37&0.65&0.17&0.78&0.84&0.70&0.84&0.55&0.17\\
VRC&0.82&0.72&-0.02&0.62&0.19&0.48&0.13&\textBF{0.85}&0.58&0.72&0.85&0.51&-0.02\\
SSWC&0.76&-0.22&0.00&0.68&0.53&0.81&0.37&0.57&0.82&0.68&0.82&0.50&-0.22\\
Dunn 31&0.73&0.60&-0.18&0.65&0.15&0.59&0.36&0.68&0.79&0.46&0.79&0.48&-0.18\\
ASSWC&0.05&-0.27&0.02&0.64&0.60&0.71&0.12&0.37&0.82&0.69&0.82&0.37&-0.27\\
PBM&0.49&-0.09&-0.10&0.27&0.56&-0.29&-0.43&0.67&0.43&0.50&0.67&0.20&-0.43\\
DB&0.57&0.58&-0.54&0.26&-0.67&0.44&0.50&-0.03&0.53&0.04&0.58&0.17&-0.67\\
\cmidrule(lr){1-11}\cmidrule(lr){12-14}
Best&0.88&0.91&0.75&0.97&0.69&0.89&0.70&0.85&0.98&0.77&0.98&0.71&0.31\\
Average&0.63&0.48&0.14&0.64&0.28&0.63&0.30&0.62&0.76&0.58&0.83&0.51&-0.09\\
Worst&0.05&-0.27&-0.54&0.26&-0.67&-0.29&-0.43&-0.03&0.43&0.04&0.58&0.17&-0.67\\
\bottomrule
\end{tabular}
\label{tab:uciresults}
\end{table}

Rather than seeking a single, general purpose favorite criterion, a more realistic approach to practical clustering evaluation is to focus on strengths of different criteria to keep a collection of reliable candidates in one's cluster analysis toolbox. From this standpoint, we argue that AUCC is a candidate to be included in this collection. In terms of reliability as assessed from the lens of robustness, it is noticeable from Table~\ref{tab:uciresults} that, in the majority of those cases in which AUCC does not provide the best evaluation, it still produces results close to the best criterion or, at least, far from the worst case.
An important aspect that also relates to reliability (and, possibly to a significant strength or weakness) of criteria is their behavior when assessing random solutions without actual cluster structure. This aspect is discussed next.    

\subsection{Expected Value}\label{expected}

AUCC (like its linearly related equivalent, Gamma Index) has the advantage that it allows for a better interpretation in terms of its expected value for chance clusterings, as shown in Corollary~\ref{corollary:aucchalf} following from Theorem~\ref{theorem:Gammazero}. We are not aware of other criteria with a theoretical characterization of its value for chance. In order to experimentally assess how the different measures behave in this regard, we ran controlled experiments with 108 synthetic datasets from~\citet{VenCamHru09,VenCamHru10}, consisting of mixtures of multivariate Gaussians with varied characteristics. The 108 datasets have 500 objects each and are obtained from three design factors comprising number of dimensions (2, 3, 4, 22, 23, or 24), number of clusters for the reference partition (2, 4, 6, 12, 14, or 16) and cluster size distribution. Regarding distribution, there are three different settings:~(i)~balanced clusters; (ii) one cluster with 10\% of the objects and the remaining objects evenly distributed among other clusters and; (iii) one cluster with 20\% (if $k^* \in \{12, 14, 16\}$) or 60\% (if $k^* \in \{2, 4, 6\}$) of the objects, and the remaining objects, again, evenly distributed among the other clusters. The term $k^*$ accounts for the actual number of clusters in the \emph{reference partition} of the dataset.

For each one of the 108 datasets we generated \emph{random partitions} as \emph{candidate clustering solutions} considering the number of clusters ($k$) in the range of 2 to $\ceil[\big]{\sqrt{n}}$, where $n$ is the number of objects, therefore $\ceil[\big]{\sqrt{500}} = 23$. The random partitions were generated considering three balances for the random clusters: (i)~balanced clusters;~(ii) one cluster with 10\% of the objects and the remaining objects evenly distributed among other clusters and; (iii) one cluster with 60\% of the objects, and the remaining objects, again, evenly distributed among the other clusters. For each dataset, balance of the random partition, and number of clusters we generated a total of 100 random clustering solutions, which were then assessed by each of the relative validity criteria considered in this study. For the sake of compactness, among the reported top performing criteria from the previous evaluation involving real data, we only show results of a single version of the Silhouette criterion (namely, the original SWC) since other variants exhibited similar behaviour.

Figure~\ref{fig:rand_eval_auc} summarizes the results. Each row of the plot depicts the results of one criterion, whereas each column accounts for a different balance of the candidate random partitions assessed. In each plot an orange line represents the average of the 100 random partitions for a given dataset. Since we ran experiments on 108 datasets, there are 108 lines per plot, plus a red line that accounts for the mean across all experiments. It is worth mentioning that for each criterion (row) the y-axis is at the very same range/scale of the criterion's value across the multiple columns.

\begin{figure}[!ht]
\centering
\includegraphics[width=.98\columnwidth,keepaspectratio=true]{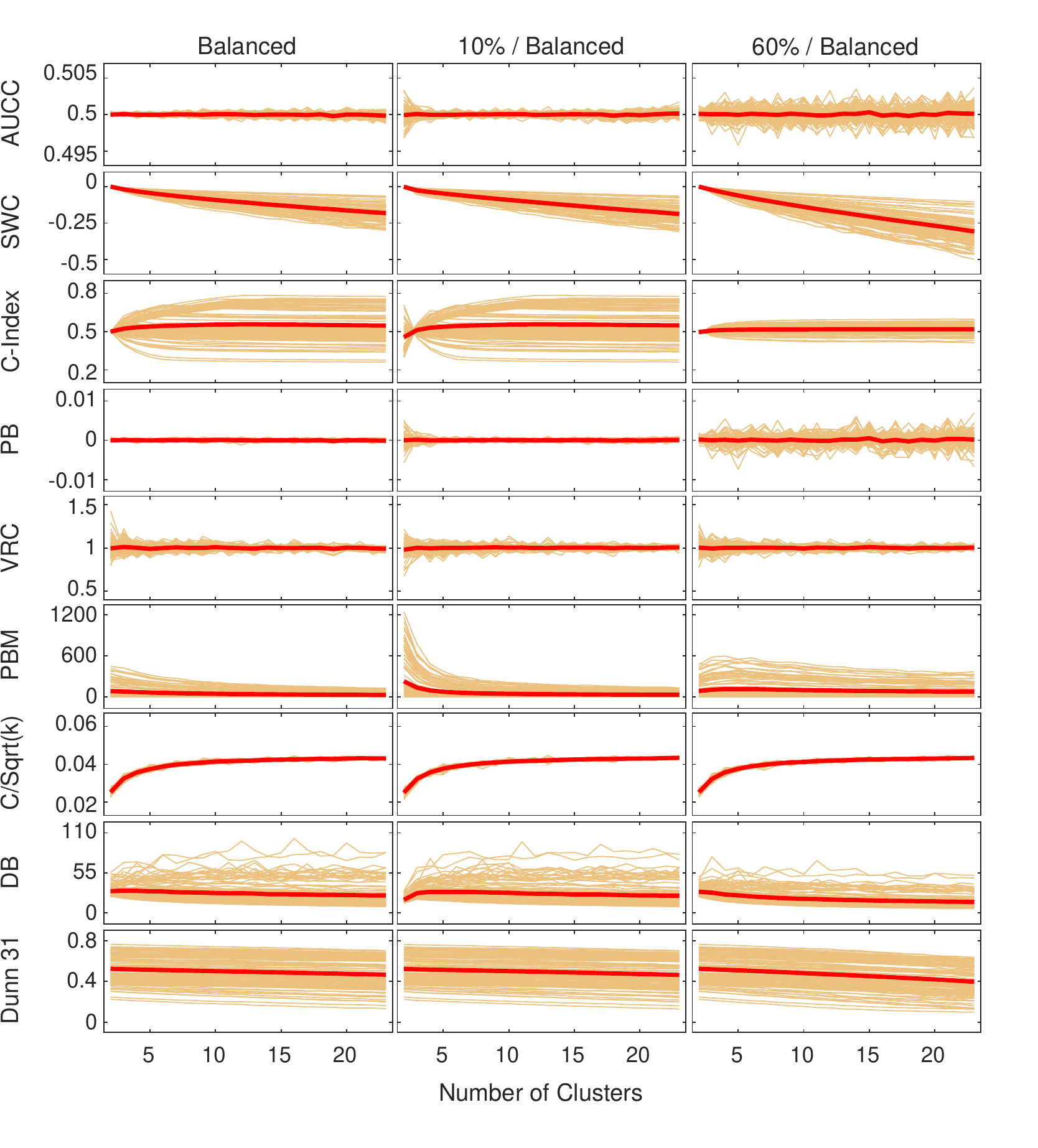}
\caption{Evaluation of random partitions of 108 synthetic datasets with varied characteristics. Each orange line corresponds to the average of 100 randomly generated partitions when evaluated in a given dataset (total of 108 lines/datasets per plot).
The number of clusters \emph{in the randomly generated partitions} is depicted in the x-axis. Results are stratified on the basis of the cluster size distribution \emph{of the generated partitions} (plot columns). The red line accounts for the overall mean.
}
\label{fig:rand_eval_auc}
\end{figure}

It can be seen that, as expected, AUCC exhibits values around 0.5 with very small variability regardless of: (i)~the number of dimensions in the dataset; (ii) the number of clusters (both in the dataset as well as in the randomly generated partitions); and (iii) the cluster size distribution (once again both in the dataset as well as in the randomly generated partitions). Besides AUCC, two other measures (PB and VRC) did not display noticeable changes in their empirical expected values for random solutions, although we are not aware of any formal proof to support this observation. Notably, other top performing measures such as Silhouette (SWC), C-Index, C/Sqrt(k) and PBM exhibited clear changes/patterns in empirical expected values as a function of the numbers of clusters and/or prominent variability across different datasets (orange lines). For SWC and C/Sqrt(k) the trend of the empirical expected value  as a function of the number of clusters (decreasing for the former and increasing for the latter) seems consistent across the different experimental settings and datasets. This is not the case for PBM and C-Index. For C-Index in particular, the empirical expected value can noticeably increase or decrease as a function of the number of clusters, depending on each particular dataset (orange line). It also varies with different size distributions in the randomly generated partitions (different columns of Figure~\ref{fig:rand_eval_auc}). 

In a practical scenario, the lack of a known constant expected value for a relative measure under a null model of random clustering solutions can impair evaluation, most noticeably when one wants to compare solutions across different numbers of clusters, because the evaluation result can be biased by the number of clusters irrespective of the quality of the assessed solutions.

\subsection{ROC Curves}

An important aspect of ROC curves is their visual interpretation, as the curves display the trade-off between sensitivity (TPR) and specificity (1-FPR) for distinct solutions. We analyzed ROC curves for the well-known Ruspini dataset (4 clusters) as well as for a simulated dataset\footnote{This dataset consists of 9 clusters, with 50 objects each, obtained from normal distributions with variance equal to 4.5, centered at $(0,0)$, $(0,20)$, $(0,40)$, $(20,0)$, $(20,20)$, $(20,40)$, $(40,0)$, $(40,20)$, and $(40,40)$.} for  candidate solutions with varying numbers of clusters (Figure~\ref{fig:clus_roc}).

\begin{figure}[!ht]
    \centering
    \includegraphics[width=\linewidth]{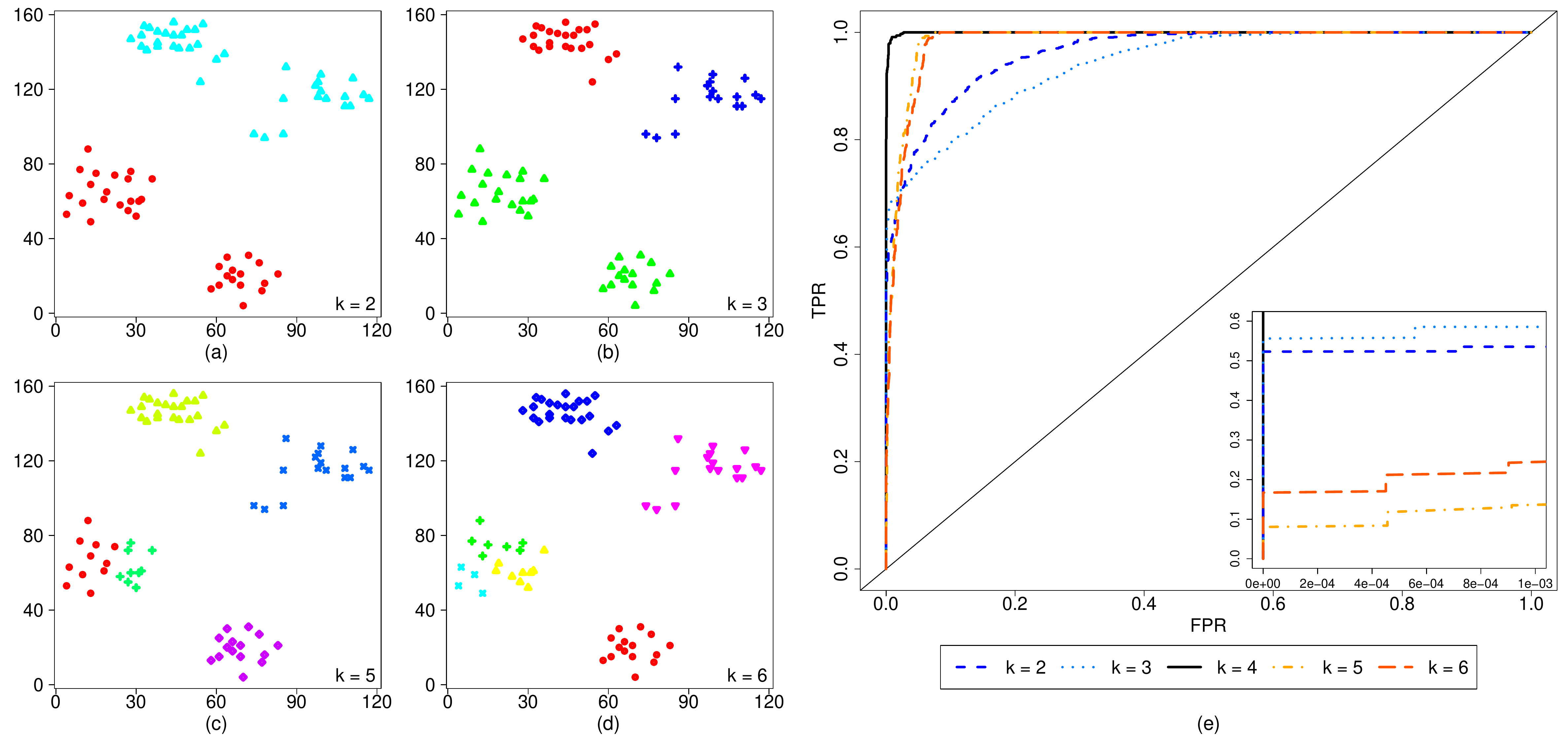}\vspace{0.2cm}
    \includegraphics[width=\linewidth]{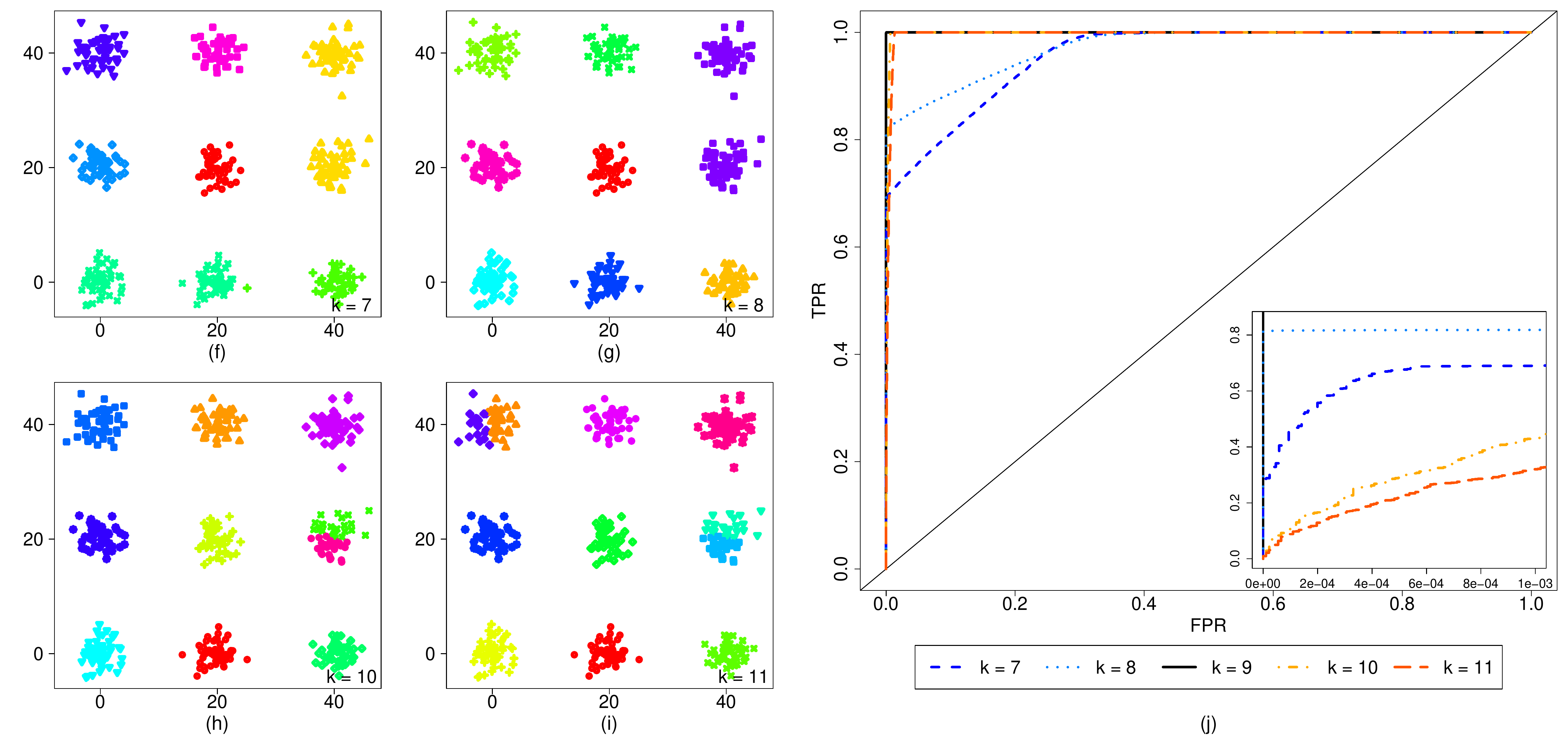}
    \caption{Clustering results and ROC Curves. From (a) to (d) we show clustering solutions with $k = 2 ,3, 5~\rm{and}~6$ for the Ruspini dataset, with (e) their corresponding ROC Curves. From (f) to~(i) we show clustering solutions with $k = 7, 8, 10~\rm{and}~11$ for a simulated dataset, with (j) their corresponding ROC Curves.~Partitions were generated with k-means. We zoom in regions at the bottom left corner to highlight the behaviour of FPR and TPR for solutions under- vs over-estimating the number of clusters for small values of the distance threshold. ROC Curves for the (omitted) partitions with the optimal number of clusters are depicted in black in both (e) and (j).}
    \label{fig:clus_roc}
\end{figure}

Solutions with the optimal number of clusters (whose partitions are not displayed in Figure~\ref{fig:clus_roc} for the sake of compactness) resulted in the best AUC values, namely, $0.9994532$ for Ruspini and $1.00$ for the simulated dataset. Solutions under-estimating the number of clusters produced lower AUC scores than solutions over-estimating the number of clusters. This is because AUCC is based on pairs of objects and a solution merging two equal sized natural clusters (under-estimation) will produce four times more errors than solutions splitting a natural cluster into two balanced halves (over-estimation). In addition, the main observed trends when comparing these solutions are as follows: starting from the bottom left of the ROC chart, as the distance (resp. similarity) threshold is raised (resp. lowered), under-estimated solutions initially tend to sustain lower FPR values while TPR increases, until a point beyond which there is mixed, alternated increments on TPR and FPR, such that very high, more steady values of TPR (approaching 1) are only reached for higher values of the distance threshold and FPR (see ROC curves in cold colors). This is because under-estimated solutions tend to merge natural clusters, so errors tend to occur from larger distances, i.e., at the lower region of the similarity rank, where there will be mixed negative (0s) and positive (1s) pairs 
whereas ideally there should only be negative ones. In contrast, over-estimated solutions observe earlier increments on FPR, which however don't prevent high, more steady values of TPR approaching 1 from also being reached earlier, i.e., for lower values of the distance threshold and FPR (see ROC curves in hot colors). This is because over-estimated solutions tend to split natural clusters into more compact sub-clusters, so errors tend to occur at smaller distances, i.e., at the upper region of the similarity rank, where there will be negative pairs (0s, introduced by the splits) mixed with (a reduced number of) positive pairs (1s), whereas ideally there should only be positive ones. This provides insights on how ROC analysis can be used to further investigate under-estimation and over-estimation in cluster solutions. 

The ROC Graph and AUCC can also provide valuable insights on the difficulty of the clustering problem itself. In Figure~\ref{fig:clus_variance} we provide an illustrative example of how different degrees of overlap between clusters can affect ROC Curves. As cluster variances increase, causing clusters to overlap, their corresponding AUCC values decrease considerably. From this perspective, scenarios in which only low to moderate AUCC evaluations are observed for a given ground-truth solution (possibly available for external clustering validation), which cannot be properly recovered despite various different clustering algorithms and parameters having been considered, may be an indicative of the intrinsic difficulties of the problem in hand. Specifically, this may be caused by the fact that the ground-truth cluster structure in question violates one of the common assumptions about clustering, namely, the assumption that clusters should be compact and separated~\citep{Eve74}. This assumption can be interpreted as within-cluster distances expected to be  smaller than between-cluster distances. It could also be the case that this is actually mostly true for some suitable underlying distance, which is not the one adopted (to assess the ground-truth or by the algorithms considered) though.


\begin{figure}[!ht]
    \centering
    \includegraphics[width=\linewidth]{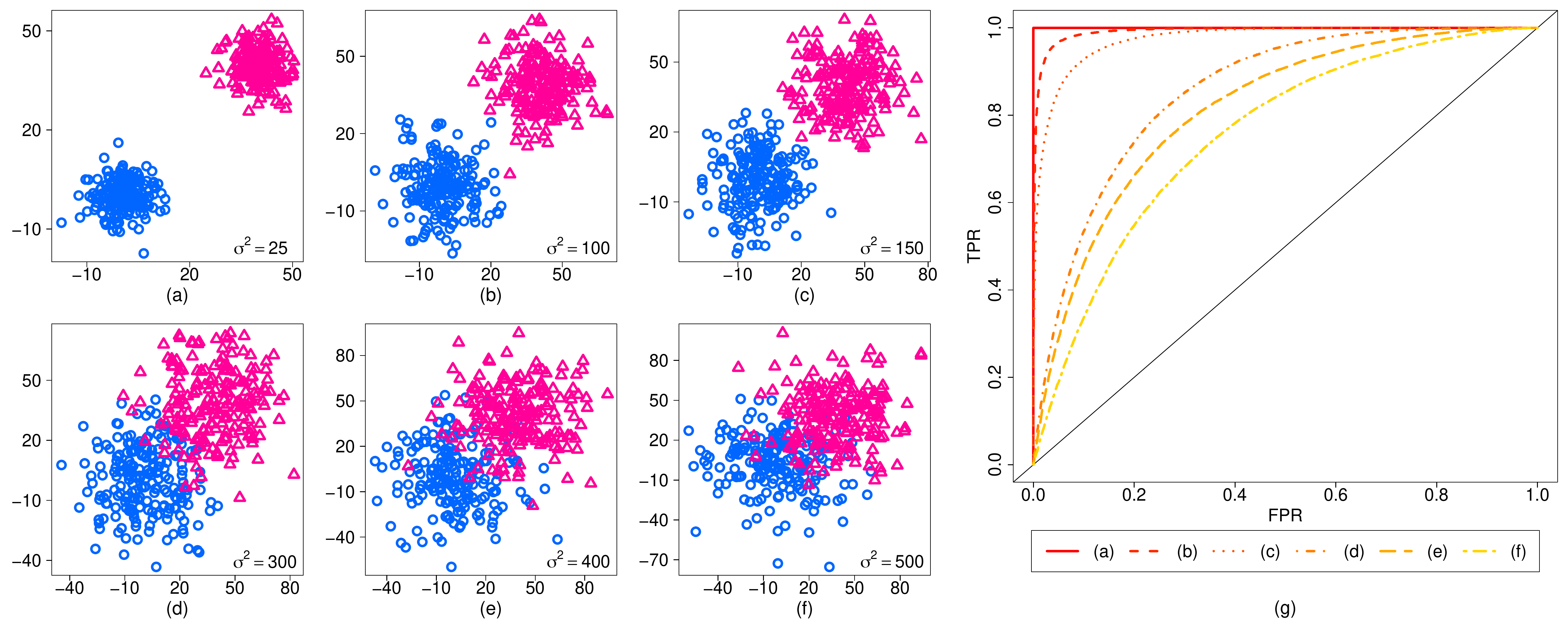}
    \caption{From (a) to (f) we depict six datasets with two clusters (200 objects per cluster) obtained from normal distributions centered at $(0,0)$ and $(40,40)$. For each dataset, cluster variances are the same for both clusters and each of their coordinates, fixed at 25, 100, 150, 300, 400, and 500, from (a) to (f). In (g) we depict the ROC Curves for the ground truth-partition of each dataset. Their respective AUC values are 1.00, 0.9923144, 0.9739683, 0.8554201, 0.8129828, and 0.7548744.}
    \label{fig:clus_variance}
\end{figure}

\section{Conclusions}\label{conc}

AUC has been extensively employed in the supervised learning domain as a valuable tool to evaluate and compare different classification models. In this work, we showed that it can also be employed in the unsupervised learning domain, more specifically, in the relative evaluation of clustering results.~In this particular setting we introduced the Area Under the Curve for Clustering, AUCC.~We theoretically showed that its expected value under a null model of equally likely random clustering solutions is 0.5, irrespective of the number or (im)balance of clusters. To our knowledge, no other relative measure has been theoretically shown to have this property.

We also showed that in the context of internal/relative clustering validation AUCC is a linear transformation of the Gamma Index from \citet{BakHub75}, for which we have also derived a theoretical expected value under a null model of random clustering partitions. In that context, we showed how ties in (dis)similarity values can be handled consistently across AUCC and Gamma so that their relationship and expected value properties are preserved. We discussed the computational complexity of these criteria and showed that AUCC represents a much more efficient algorithmic way to implement Gamma. We also showed that a visual inspection of the AUCC provides insights on clustering solutions producing under-estimated and over-estimated solutions, as well as on the difficulty of the clustering problem.

In addition to its theoretical, computational and visual appeals, AUCC exhibited very competitive results in our experimental evaluations using well-known classification benchmark datasets. These results need, however, to be taken with a grain of salt. In fact, as argued e.g. by \citet{Hennig2015b}, cluster analysis can have different aims in different contexts, and is not necessarily or only about finding a unique ``true'' clustering. Given a diverse collection of candidates, finding a clustering that captures well the structure in the data according to a given notion of similarity may be of interest for reasons other than recovery performance on benchmark datasets with ground-truth, and the AUCC/Gamma Index is one way of doing it. In summary, since there is no free lunch in clustering evaluation, we believe that AUCC can be a useful additional tool in an analyst's toolbox as a standalone measure or in combination with other measures. Indeed, the combination of relative validity measures has gained attention in the literature recently~\citep{VenJasCam13,JasMouFurCam15,Kim2018}.

\bibliographystyle{chicago}
\bibliography{library.bib,additional_library.bib}   

\end{document}